\documentclass{article}

\usepackage{iclr2022_conference,times}

\usepackage{amsmath,amsfonts,bm}

\def\eqref#1{equation~\ref{#1}}

\def\1{\bm{1}}

\DeclareMathAlphabet{\mathsfit}{\encodingdefault}{\sfdefault}{m}{sl}
\SetMathAlphabet{\mathsfit}{bold}{\encodingdefault}{\sfdefault}{bx}{n}

\usepackage{url}
\usepackage{amsfonts, amssymb, amsmath, amsthm, graphicx, color, tikz, booktabs, multirow, bm, cases, mathtools,  mathrsfs, subcaption, pgfplots, csvsimple,algorithm,algorithmic,wrapfig}
\usetikzlibrary{intersections, pgfplots.fillbetween, shapes, arrows, positioning, decorations.markings}
\definecolor {processblue}{cmyk}{0.96,0,0,0}
\tikzstyle{int}=[draw, fill=blue!20, minimum size=2em]
\tikzstyle{init} = [pin edge={to-,thin,black}]
\usepackage{pgfplotstable}
\usepackage{pgfplots}
\pgfplotsset{compat=1.14}
\usepackage{hyperref, graphics}
\newtheorem{lem}{Lemma}
\newtheorem{rem}{Remark}

\newtheorem{thm}{Theorem}

\newtheorem{cor}{Corollary}

\newcommand{\ex}[2]{{\ifx&#1& \mathbb{E} \else \underset{#1}{\mathbb{E}} \fi \left[#2\right]}}

\usepackage[toc,page,header]{appendix}
\usepackage{minitoc}

\doparttoc 
\faketableofcontents %

\title{On the Generalization of Models Trained with SGD: Information-Theoretic Bounds and Implications}

\author{%
  Ziqiao Wang
    \\
  University of Ottawa \\
  \texttt{zwang286@uottawa.ca} \\
  \And
  Yongyi Mao \\
  University of Ottawa \\
  \texttt{ymao@uottawa.ca} \\
}

\iclrfinalcopy
\begin{document}

\maketitle

\begin{abstract}

This paper
follows up on a recent work of \citet{Neu2021InformationTheoreticGB} and 
presents some new information-theoretic upper bounds for the generalization error of machine learning models, such as neural networks, trained with SGD.  
We apply these bounds to analyzing the generalization behaviour of linear and two-layer ReLU networks. Experimental study  of these bounds provide some insights on the SGD training of neural networks. They also point to a new and simple regularization scheme which we show performs comparably to the current state of the art. 
\end{abstract}

\section{Introduction}


The observation that high capacity deep neural networks trained with mini-batched stochastic gradient descent, referred to SGD in this paper, tend to generalize well \citep{ZhangBHRV17} contradicts the classical wisdom in statistical learning theory
(e.g., \cite{Vapnik1998}  
) and has stimulated intense research interest in understanding the generalization behaviour of modern neural networks.



In this direction, generalization bounds for over-parameterized neural networks are obtained \citep{allen2019convergence,bartlett2017spectrally,neyshabur2015norm,neyshabur2018pac,neyshabur2018role,arora2018stronger,arora2019fine} and a curious ``double descent'' phenomenon is observed and analyzed \citep{belkin2019reconciling,nakkiran2019deep,yang2020rethinking}.
 Built on a connection between stability and generalization \citep{bousquet2002stability}, a stability-based bound is first presented in  \cite{hardt2016train}, followed  by a surge of research effort exploiting similar approaches  \citep{london2017pac,chen2018stability,feldman2019high,lei2020fine,bassily2020stability}. 
Information-theoretic bounding techniques established recently 
\citep{russo2016controlling,russo2019much,xu2017information,asadi2018chaining,bu2020tightening,SteinkeZ20,dwork2015generalization,bassily2018learners,asadi2020chaining,hafez2020conditioning,zhou2020individually} have also demonstrated great power in analyzing SGD-like algorithms.  For example, \cite{pensia2018generalization} is the first to utilize information-theoretic bound in analyzing the generalization ability of SGLD \citep{gelfand1991recursive,welling2011bayesian}.
The bound was subsquently improved by \cite{negrea2019information,haghifam2020sharpened,rodriguez2020random,wang2021learning}.
Inspired by the work of \cite{pensia2018generalization}, \cite{Neu2021InformationTheoreticGB} presents an information-theoretic analysis of the models trained with SGD. The analysis of \cite{Neu2021InformationTheoreticGB} constructs an auxiliary weight process parallel to SGD training and upper-bounds the generalization error through this auxiliary process. 

Another line of research connects the generalization of neural networks with the flatness of loss minima \citep{hochreiter1997flat} found by SGD or its variant \citep{keskar2017large,dinh2017sharp,DziugaiteR17,NeyshaburBMS17,ChaudhariCSLBBC17,jastrzkebski2017three,jiang2019fantastic, zheng2020regularizing,foret2020sharpness}. This understanding has led to the discovery of new SGD-based training algorithms for improved generalization. For example, in a concurrent development by \cite{zheng2020regularizing} and \cite{foret2020sharpness}, a local ``max-pooling'' operation is applied to the loss landscape prior to the SGD updates. This approach, referred to as AMP \citep{zheng2020regularizing} or SAM \citep{foret2020sharpness}, 
is shown to make SGD favor flatter minima and 
achieve the state-of-the-art performance among various competitive regularization schemes.


In this paper, we focus on investigating the generalization of machine learning models trained with SGD. Although we are primarily motivated by the curiosity to understanding neural networks, the results of this paper in fact apply broadly to any model trained with SGD.

This work
follows the same construction of the auxiliary weight process in \cite{Neu2021InformationTheoreticGB} and develops
upper bounds of generalization error that extend the work of \cite{Neu2021InformationTheoreticGB}.
Like those in \cite{Neu2021InformationTheoreticGB}, the bounds we obtain can be decomposed into two terms, one measuring the impact of training trajectories (``the trajectory term'') and the other measuring the impact of the flatness of the found solution (``the flatness term\footnote{We note this term is correlated with flatness rather than precisely measures the flatness.}'').
Having an identical flatness term as that in \cite{Neu2021InformationTheoreticGB}, empirical evidence hints that our bounds have a improved trajectory term. Figure \ref{fig:bound-compare} shows an experimental comparison of the trajectory term in  our bound (Theorem \ref{thm:re-neu-bound}) with that in the bound in \cite{Neu2021InformationTheoreticGB} (re-stated as Lemma \ref{lem:neu-bound} in this paper)
for two neural network models. The trajectory terms are compared in a deterministic setting for two different values of $\sigma$ (the variance parameter of the noise in the auxiliary weight process, details given in appropriate context and in Appendix \ref{appendix:sec-more-compare}). 

The trajectory term in the bounds of \cite{Neu2021InformationTheoreticGB} accumulates two terms over training steps: {\em local gradient sensitivity}, measuring the sensitivity of the SGD gradient signal to weight perturbations, and {\em gradient dispersion}\footnote{The quantity is often referred to as gradient variance in the literature, 
but we prefer ``dispersion'' to ``variance'' so as to better comply with the mathematical conventions and avoid possible confusion.}, measuring the extent to which the gradient signal  spreads around its mean. Usually gradient dispersion vanishes with training iterations but local gradient sensitivity does not. Our improvement over \cite{Neu2021InformationTheoreticGB} is achieved by removing the local sensitivity term and 
involving the logarithmic of a revised gradient dispersion. Although our new definition of gradient dispersion is in general larger than that in \cite{Neu2021InformationTheoreticGB}, as long as the number of iterations is not too small, the fact that our gradient dispersion also vanishes with training allows our bounds to be tighter than \cite{Neu2021InformationTheoreticGB}.

\begin{figure}[ht!]
\begin{subfigure}[b]{0.24\columnwidth}%
\centering%
\captionsetup{font=small}%
\scalebox{0.35}{
\begin{tikzpicture}
\begin{axis}
[legend style={nodes={scale=1.25, transform shape}},
legend cell align={left},
legend pos=north west,
xlabel=epoch,
xmin=0,
    xmax=50,
ylabel=Trajectory Term.
]
\addplot[line width=2pt, color=red] table [y=Ours2, x=epoch]{data/MLP_var_sen_std05.txt};
\addlegendentry{Ours}
\addplot[line width=2pt, color=orange] table [y=Neu, x=epoch]{data/MLP_var_sen_std05.txt};
\addlegendentry{\cite{Neu2021InformationTheoreticGB}}
\end{axis}
\end{tikzpicture}
}
\caption{$\sigma=10^{-5}$ (MNIST)}%
\end{subfigure}
\begin{subfigure}[b]{0.24\columnwidth}%
\centering%
\captionsetup{font=small}%
\scalebox{0.35}{
\begin{tikzpicture}
\begin{axis}
[legend style={nodes={scale=1.25, transform shape}},
legend cell align={left},
legend pos=north west,
xlabel=epoch,
xmin=0,
    xmax=100,
ylabel=Trajectory Term.
]
\addplot[line width=2pt, color=red] table [y=Ours2, x=epoch]{data/MLP_var_sen_std06.txt};
\addlegendentry{Ours}
\addplot[line width=2pt, color=orange] table [y=Neu, x=epoch]{data/MLP_var_sen_std06.txt};
\addlegendentry{\cite{Neu2021InformationTheoreticGB}}
\end{axis}
\end{tikzpicture}
}
\caption{$\sigma=10^{-6}$ (MNIST)}%
\end{subfigure}
\begin{subfigure}[b]{0.24\columnwidth}%
\centering%
\captionsetup{font=small}%
\scalebox{0.35}{
\begin{tikzpicture}
\begin{axis}
[legend style={nodes={scale=1.25, transform shape}},
legend cell align={left},
legend pos=north west,
xlabel=epoch,
xmin=0,
    xmax=50,
ylabel=Trajectory Term.
]
\addplot[line width=2pt, color=red] table [y=Ours2, x=epoch]{data/CNN_var_sen_std05.txt};
\addlegendentry{Ours}
\addplot[line width=2pt, color=orange] table [y=Neu, x=epoch]{data/CNN_var_sen_std05.txt};
\addlegendentry{\cite{Neu2021InformationTheoreticGB}}
\end{axis}
\end{tikzpicture}
}
\caption{$\sigma=10^{-5}$ (CIFAR10)}%
\end{subfigure}
\begin{subfigure}[b]{0.24\columnwidth}%
\centering%
\captionsetup{font=small}%
\scalebox{0.35}{
\begin{tikzpicture}
\begin{axis}
[legend style={nodes={scale=1.25, transform shape}},
legend cell align={left},
legend pos=north west,
xlabel=epoch,
xmin=0,
    xmax=100,
ylabel=Trajectory Term.
]
\addplot[line width=2pt, color=red] table [y=Ours2, x=epoch]{data/CNN_var_sen_std06.txt};
\addlegendentry{Ours}
\addplot[line width=2pt, color=orange] table [y=Neu, x=epoch]{data/CNN_var_sen_std06.txt};
\addlegendentry{\cite{Neu2021InformationTheoreticGB}}
\end{axis}
\end{tikzpicture}
}
\caption{$\sigma=10^{-6}$ (CIFAR10)}%
\end{subfigure}
\caption{Comparison of the trajectory term between our bound (Theorem \ref{thm:re-neu-bound}) and the bound in \cite{Neu2021InformationTheoreticGB}.
(a)(b) MLP trained on MNIST. (c)(d) CNN trained on CIFAR-10.
}
\label{fig:bound-compare}
\end{figure}
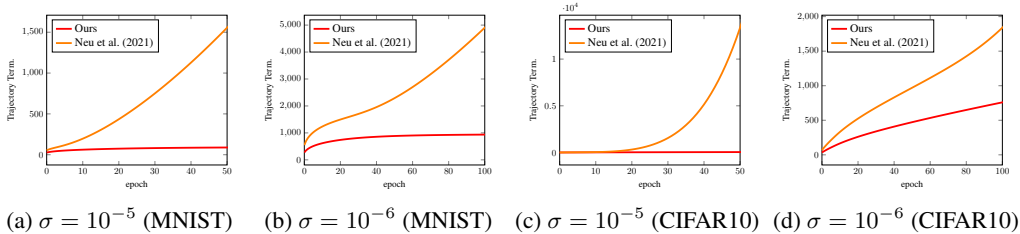

We also provide an application of our bounds in analyzing the generalization behaviour of linear and two-layer ReLU networks, where we show that the activation state in ReLU networks plays an important role in generalization.

It is remarkable that removing the local gradient sensitivity term makes our bound have a simple closed form (after optimizing the noise variances), much easier to evaluate. We empirically validate the derived bounds, and provide various insights pertaining to the generalization behavior of models trained with SGD.
For example, gradient dispersion is seen to reveal a double descent phenomenon with respect to training epochs, where the valley in the double descent curve appears to mark the great divide between the ``generalization regime'' and the ``memorization regime'' of training. Furthering from this observation, we also show that it is possible to reduce the memorization effect by dynamically clipping the gradient and reducing its dispersion.


Our bounds also inspire a natural and simple solution to alleviate generalization error. Specifically, we propose a new training scheme, referred to as {\em Gaussian model perturbation} (GMP),  aiming at reducing the flatness term of the bounds. This scheme effectively applies a local ``average pooling'' to the empirical risk surface prior to SGD, greatly resembling the ``max-pooling'' approach adopted in AMP \citep{zheng2020regularizing}. We demonstrate experimentally that GMP achieves a competitive performance with the current art of regularization schemes.

Proofs, additional discussions and experimental results are presented in Appendices.
\paragraph{Other Related Literature}
Gradient dispersion is mostly studied from optimization perspectives \citep{bottou2018optimization,roux2012stochastic,johnson2013accelerating,wen2020empirical,faghri2020study}. Prior to this work,  only a few works relate gradient dispersion with the generalization behaviour of the networks. In \cite{Neu2021InformationTheoreticGB} and \cite{wang2021learning}, gradient dispersion also appears in the generalization bounds. 
and there are limited studies characterizing the generalization performance and the gradient variance. Particularly, 
in \cite{jiang2019fantastic}, gradient dispersion is argued to capture a notion of “flatness” of the local minima of the loss landscape, thereby correlating with generalization.
Injecting noise in the training process has been proposed in various regularization schemes  \citep{bishop1995training,CamutoWSRH20,camuto2021asymmetric,srivastava2014dropout,wei2020implicit}.
But unlike GMP derived in this paper, where noise is injected to the model parameters, noise in those schemes is injected either to the training data or to the network activation. In addition, noise can also be implicitly  injected during training, for example, by adjusting the momentum terms \cite{xie2021positive}. The training objective in \cite{xie2021artificial} is similar to our GMP but their NVRM-SGD minimizes a region directly while our GMP considers a trade-off between the empirical loss and the flatness term.
Gradient clipping is a common technique for
preventing gradient exploding (see, e.g., \cite{merity2018regularizing,peters2018deep}). 
This technique is also used in \cite{zhang2019gradient} to accelerate training.
In this paper, gradient clipping is used to investigate and control the impact of gradient dispersion on generalization error.

\section{Preliminaries}
\label{sec:pre}
\paragraph{Population Risk, Empirical Risk and Generalization Error}
Unless otherwise noted, 
a random variable will be denoted by 
a capitalized letter (e.g., $Z$), and  its realization denoted by the corresponding lower-case letter (e.g. $z$).
Let $\mathcal{Z}$ be the instance space of interest and $\mu$ be an unknown distribution on $\mathcal {Z}$, specifying random variable $Z$.  
Let ${\mathcal W}\subseteq \mathbb{R}^d$ be the space of hypotheses. 
Suppose that a training sample $S=(Z_1, Z_2, \dots, Z_n)$ is drawn i.i.d. from $\mu$ and that a stochastic learning algorithm $\mathcal{A}$ takes $S$ as its input and outputs a hypothesis $W\in {\mathcal W}$ according to some  conditional distribution 
$P_{W|S}$ mapping ${\mathcal Z}^n$ to ${\cal W}$. 
Let $\ell: \mathcal{W}\times\mathcal{Z}\rightarrow \mathbb{R}^{+}$ be a loss function, where 
$\ell (w, z)$ 
measures the ``unfitness'' or ``error'' of any $z\in {\mathcal Z}$ with respect to a hypothesis $w\in {\mathcal W}$. 
The population risk, for any $w\in {\mathcal W}$,  is defined as
\[
L_\mu(w) \triangleq \mathbb{E}_{Z\sim \mu}[\ell(w,Z)].
\]
The goal of learning is to find a hypothesis $w$ that minimizes the population risk. But since $\mu$ is only partially accessible via the sample $S$, in practice, we instead turn to the empirical risk, defined as
\[
L_S(w) \triangleq \frac{1}{n}\sum_{i=1}^n \ell(w,Z_i).
\]
The expected generalization error of the learning algorithm $\mathcal{A}$ is then defined as 
\[
{\rm gen}(\mu,P_{W|S})\triangleq\mathbb{E}_{W,S}[L_\mu(W)-L_S(W)],
\]
where the expectation is taken over the joint distribution of $(S,W)$ (i.e., $\mu^n\otimes P_{W|S}$).

Throughout this paper, 
we take $\ell$ as a continuous function (adopting the usual notion ``surrogate loss'' 
\citep{shalev2014understanding}
).
Additionally, we assume that $\ell$ is differentiable almost everywhere with respect to both $w$ and $z$. Furthermore we assume that $\ell(w, Z)$ is $R$-subgaussian\footnote{A random variable $X$ is $R$-subgaussian if for any $\rho$, $\log {\mathbb E} \exp\left( \rho \left(
X- {\mathbb E}X
\right) \right) \le \rho^2R^2/2$.}
for any $w\in\mathcal{W}$. Note that a bounded loss is guaranteed to be subgaussian.
Let $I(X; Y)$ denote the mutual information 
\citep{cover2012elements}
between any pair of random variables $(X, Y)$. 
The following result is known.

\begin{lem}[{\citet[Theorem~1.]{xu2017information}}]
Assume the loss $\ell(w, Z)$ is $R$-subgaussian for any $w \in \mathcal{W}$. 
The generalization error of  ${\cal A}$ is bounded by
\[
|{\rm gen}(\mu,P_{W|S})|\leq \sqrt{\frac{2R^2}{n}I(W;S)},
\]
\label{lem:xu's-bound}
\end{lem}


\paragraph{Stochastic Gradient Descent}
We now restrict the learning algorithm ${\cal A}$ to be the mini-batched SGD algorithm for empirical risk minimization. For each training epoch, the dataset $S$ is randomly split into $m$ disjoint mini-batches, each having size $b$, namely, $n=mb$. Based on each batch, one parameter update is performed. Specifically, let $B_t$ denote the batch used for the $t^{\rm th}$ update.  Define 
\[
g(w,B_t) \triangleq \frac{1}{b}\sum_{z\in B_t}\nabla_w {\ell}(w, z),
\]
namely, 
$g(w,B_t)$ is the
average gradient computed for the batch $B_t$ with respect to parameter $w$. The rule for the $t^{\rm th}$ parameter update is then
\[
    W_t \triangleq W_{t-1} - \lambda_t g(W_{t-1}, B_t),
\]
where $\lambda_t$ is the learning rate at the step $t$. The initial parameter setting $W_0$ is assumed to be drawn from the zero-mean spherical Gaussian ${\cal N}(0, \sigma_0^2{\bf I}_d)$ with variance $\sigma_0^2$ in each dimension. We will assume that the SGD algorithm stops after $T$ updates and outputs $W_T$ as the learned model parameter.

Given the training sample $S$, let $\xi$ govern the randomness in the  sequence $(B_1, B_2, \ldots, B_T)$ of batches. For the simplicity of notion, we will fix the configuration of $\xi$. That is, we will assume a fixed ``batching trajectory'', or a fixed way to shuffle the example indices $\{1, \ldots, n\}$ and divide them into $m$ batches in each epoch. The presented generalization bounds of this paper can be extended to the case where the batching trajectory is uniformly random (as we set up above). This merely involves averaging over all batching trajectories or taking expectation over $\xi$.

\paragraph{Auxiliary Weight Process}
We now associate with the SGD algorithm an auxiliary weight process $\{\widetilde{W}_t\}$.
Let $\sigma_1, \sigma_2, \ldots, \sigma_T$ be a sequence of positive real numbers. 
Define
\[
    \widetilde{W}_0  \triangleq  W_0,\quad {\rm and}~~
    \widetilde{W}_t  \triangleq  \widetilde{W}_{t-1} - \lambda_t g(W_{t-1}, B_t) + N_t, ~{\rm for} ~t>0,
\]
where $N_t \sim \mathcal{N}(0, \sigma_t^2\mathbf{I}_d)$ is a Gaussian noise.  The relationship between this auxiliary weight process 
$\{\widetilde{W}_t\}$ and the weight process $\{W_t\}$ in SGD is shown in the Bayesian network 
below.
\begin{center}
{\small 
    \begin{tabular}{clclclclclc}
          &  & $N_1$  &  & $N_2$&   & $\cdots$ &  & $N_{T-1}$&  & $N_{T}$ \\
          & &$\downarrow$ & & $\downarrow$ & & & &$\downarrow$& &$\downarrow$\\
         $\widetilde{W}_0$ &  $\rightarrow$ & $\widetilde{W}_1$ & $\rightarrow$ & $\widetilde{W}_2$& $\rightarrow$  & $\cdots$& $\rightarrow$ & $\widetilde{W}_{T-1}$& $\rightarrow$ & $\widetilde{W}_{T}$ \\
          $\parallel$ & $\nearrow$&  & $\nearrow$&  &$\nearrow$ & & & &$\nearrow$ &\\
         $W_0$& $\rightarrow$  & $W_1$ & $\rightarrow$ & $W_2$& $\rightarrow$ & $\cdots$& $\rightarrow$ & $W_{T-1}$& $\rightarrow$ & $W_{T}$\\
    \end{tabular}
}
\end{center}

Let $\Delta_t=\sum_{\tau=1}^tN_\tau$. Noting that the weight updates in $\{\widetilde{W}_t\}$ uses the same gradient signal as that used in $\{W_t\}$ (which depends on $W_{t-1}$ not $\widetilde{W}_{t-1}$), it is immediate that
$   \widetilde{W}_t = W_t + \Delta_t.$
Note that this auxiliary process follows the same construction as \cite{Neu2021InformationTheoreticGB}, which we will use to study the generalization error of SGD.
To that end, let's define \textit{gradient dispersion} by
\begin{align}
\label{eq:def-gradient-dispersion}
    \mathbb{V}_t(w)\triangleq \ex{}{||g(w,B_t)-\ex{}{\nabla_w\ell(W,Z)}||_2^2},
\end{align}
where the inside expectation is taken over $(W,Z)\sim\mu\otimes P_{W|Z}$. 
For a given sample $s \in {\cal Z}^n$, define
\begin{align*}
    \gamma(w,s)\triangleq \mathbb{E}\left[L_s(w+ \Delta_T)-L_s(w)\right],
\end{align*}
where the expectation is taken over $\Delta_T$ and $L_s(w)$ is the empirical risk of $s$ at parameter $w$.

In the remainder of the paper, let $S'$ denote another sample drawn from $\mu^n$, independent of all other random variables. The main generalization bound in \cite{Neu2021InformationTheoreticGB} is re-stated below. 

\begin{lem}[{\cite[Theorem~1.]{Neu2021InformationTheoreticGB}}]

The generalization error of SGD is upper bounded by
\begin{align*}
|\mathrm{gen}(\mu,P_{W_T|S})|\leq 2\sqrt{\frac{2R^2}{n}\sum_{t=1}^T\frac{\lambda_t^2 }{\sigma_t^2}\ex{}{\Psi(W_{t-1})+ \mathbb{\widetilde{V}}_t(W_{t-1})}}+\left|\ex{}{\gamma(W_T,S)-\gamma(W_T,S')}\right|,
\end{align*}
where $\Psi(w_{t-1}) \triangleq \ex{}{||\ex{}{\nabla_w\ell(w_{t-1},Z)}-\ex{}{\nabla_w\ell(w_{t-1}+\zeta,Z)}||^2_2}$, $\zeta\sim\mathcal{N}(0,\sum_{i=1}^{t-1}\sigma^2_i \mathrm{I}_d)$ and $\mathbb{\widetilde{V}}_t(w)\triangleq \ex{}{||g(w,B_t)-\ex{}{\nabla_w\ell(w,Z)}||_2^2}$.
\label{lem:neu-bound}
\end{lem}
The term $\Psi(w_{t-1})$ in the bound is referred to as ``local gradient sensitivity'' in \cite{Neu2021InformationTheoreticGB}.
Note that the inside expectation of $\mathbb{\widetilde{V}}_t(w)$ is taken over $Z\sim\mu$ instead of $(W,Z)\sim\mu\otimes P_{W|Z}$. Thus, ${\mathbb{\widetilde{V}}_t(w)}$ is in general not worse than our gradient dispersion ${\mathbb{{V}}_t(w)}$ for a fixed $w$ (see Figure \ref{fig:CompareMore} (d) in Appendix \ref{appendix:sec-more-compare}). However, the difference between these two terms is very small when $W$ is close to local minima due to the tiny gradient norm. 
In addition, the definition of our gradient dispersion is implicitly used in a bound in Section 5.2 of \cite{Neu2021InformationTheoreticGB}, which can be regarded as a weaker version of Eq. \ref{ineq:sgd-bound} in our Theorem \ref{thm:re-neu-bound}.

\section{New Generalization Bounds for SGD}
\label{sec:theory}

Removal of the local sensitivity term  $\Psi(w_{t-1})$ requires invoking a special instance of the HWI inequality \cite[Lemma~3.4.2]{raginsky2018concentration}, which we first state.

\begin{lem}
Let X and Y be two random vectors in $\mathbb{R}^d$, and let $N \sim \mathcal{N}(0,\mathbf{I}_d)$ be independent of $(X, Y )$. Then, for any $t,t' > 0$,
$
    \mathrm{D_{KL}}(P_{X+\sqrt{t}N}||P_{Y+\sqrt{t'}N})\leq \frac{1}{2t'}\mathbb{E}\left[||X-Y||_2^2\right]+\frac{d}{2}(\log\frac{t'}{t}+\frac{t}{t'}-1),
$
\label{lem:HWI}
where $\mathrm{D_{KL}}$ is the KL divergence.
\end{lem}



Note that \cite{Neu2021InformationTheoreticGB} also makes uses of a similar lemma in their revision. However, the development in \cite{Neu2021InformationTheoreticGB} requires constructing a ghost auxiliary weight process in which an independent noise perturbation $\Delta_t'$ is introduced. Due to the mismatch between $\Delta_t$ and $\Delta_t'$, the local gradient sensitivity term $\Psi(W_{t-1})$ appears in their bound. In this paper, we waive this $\Delta_t'$ by using the following lemma.
\begin{lem}
\label{lem:keylemma}
Let random variables $X$,$Y$ and $\Delta$ be independent of $N\sim \mathcal{N}(0,\mathbf{I}_d)$. Then for any $\sigma>0$, any ${\mathbb R}^d$-valued function $f$, and any random variable $\Omega\in\mathbb{R}^d$ that is a function of $Y$, we have
\[
I(f(Y+\Delta,X)+\sigma N;X|Y)\leq \frac{d}{2}\ex{}{\log\left(\frac{\ex{}{||f(Y+\Delta,X)-\Omega||^2}}{d\sigma^2}+1\right)}.
\]
\end{lem}

Note that the outside expectation is taken over $Y$ and the inside expectation is taken over $(X,\Delta)$. Then, exploiting Lemma \ref{lem:keylemma} by letting $X=S$, $Y=\widetilde{W}$ and $\Omega=\ex{}{g(\widetilde{w}-\Delta,Z)}\triangleq\ex{}{\nabla\ell(\widetilde{w}-\Delta,Z)}$ will enable us to have the bound below.
\begin{thm}
\label{thm:tightest-bound}
The generalization error of SGD is upper bounded by
\begin{align*} 
\sqrt{\frac{R^2d}{n}\sum_{t=1}^T\ex{}{\log\left(\frac{\lambda_t^2\ex{}{||g(W_{t-1},B_t)-\ex{}{g(W_{t-1},Z)}||^2}}{d\sigma_t^2}+1\right)}}+\left|\ex{}{\gamma(W_T,S)-\gamma(W_T,S')}\right|.
\end{align*}
\end{thm}

This bound is strictly tighter than the bound in Lemma \ref{lem:neu-bound}. In fact, from Theorem \ref{thm:tightest-bound}, one can recover Lemma \ref{lem:neu-bound} with a smaller constant factor (see Appendix \ref{appendix:recover-bound} for more details). 

To completely remove $\Psi(w_{t-1})$ and to obtain a closed form of the optimal bound, we will let $\Omega=\ex{}{g(W,Z)}$, then
Lemma \ref{lem:keylemma} gives us a crisp way to have the following upper bound that is independent of the distribution of $\Delta$.

\begin{lem}
Let $G_t = -\lambda_t g(W_{t-1},B_t)$. Then,
$
    I(G_t+N_t;S|\widetilde{W}_{t-1})\leq \frac{1}{2}d\log\left(\frac{\lambda_t^2\ex{}{\mathbb{V}_t(W_{t-1})}}{d\sigma_t^2}+1\right).
$
\label{lem:sample-cmi-var}
\end{lem}
In this lemma, the mutual information $I(G_t+N_t;S|\widetilde{W}_{t-1})$ roughly indicates the degree by which the SGD's updating signal $G_t$ (smoothed with noise) depends on the training sample $S$, when $B_t$ is used for computing the gradient. When this dependency is strong (giving rise to a high value of the mutual information), the model conceivably
tends to overfit the training sample. This lemma suggests that the strength of this dependency can be upper-bounded by the expected gradient dispersion at the current weight configuration. In our experiments, we will estimate the expected gradient dispersion and validate this intuition. 




We are now in a position to state our main theorem. 
\begin{thm}
The generalization error of 
SGD is upper bounded by
\begin{align}
|\mathrm{gen}(\mu,P_{W_T|S})|\leq \sqrt{\frac{R^2d}{n}\sum_{t=1}^T\log\left(\frac{\lambda_t^2\ex{}{\mathbb{V}_t(W_{t-1})}}{d\sigma_t^2}+1\right)}+\left|\ex{}{\gamma(W_T,S)-\gamma(W_T,S')}\right|.
\label{ineq:sgd-bound}
\end{align}
Further, assume $L_\mu(w_T)\leq \mathbb{E}_{\Delta}\left[L_\mu(w_T+ \Delta_T)\right]$, $\ell$ is twice differentiable, and
$\sigma^2_t$ is 
independent of $t$. Denote by $\mathrm{H}_{W_T}$ the Hessian of the loss with respect to $W_T$ and let  $\mathrm{Tr}(\cdot)$ denote trace.
Then
\begin{align}
    \mathrm{gen}(\mu,P_{W_T|S})\leq\frac{3}{2}\left(\sum_{t=1}^T\frac{R^2\lambda_t^2T}{n} \ex{}{\mathbb{V}_t(W_{t-1})}\ex{}{\mathrm{Tr}\left({\mathrm{H}_{W_T}(Z)} \right)}\right)^{\frac{1}{3}}.
    \label{eq:optimal-bound}
\end{align}
\label{thm:re-neu-bound}
\end{thm}

\begin{rem}
With a single draw of $S$ and under the deterministic setting (i.e., with fixed weight initialization and batching trajectory), removing $\Psi(W_{t-1})$ will make the bound much tighter as shown in Figure \ref{fig:bound-compare}. This remarkable improvement should come at no surprise, since $\Psi(W_{t-1})$ monotonically increases with training epochs (noting that $\Psi(W_{t-1})$ has the cumulative variance $\sum_{i=1}^{t-1}\sigma^2_i \mathrm{I}_d$
) while $\mathbb{V}_t(W_{t-1})$ appears decreasing (see Figure \ref{fig:acc-var-mlp}). 
Figure \ref{fig:bound-compare} also indicates that if the noise variance $\sigma_t$ is made small, one expects  the gradient dispersion to dominate the trajectory term in Lemma \ref{lem:neu-bound}. However,  the factor $1/\sigma_t^2$ will become too large, make the bound loose. More discussions about the stochastic setting with multiple draws of $S$ are deferred to Appendix \ref{appendix:sec-more-compare}.
\end{rem}

\begin{rem}
The condition $L_\mu(w_T)\leq \mathbb{E}_{\Delta_T}\left[L_\mu(w_T+ \Delta_T)\right]$
indicates that the perturbation does
 not decrease the population risk. This is also assumed in \cite{foret2020sharpness} in the derivation of a PAC-Bayesian
 generalization bound.
\end{rem}


In the bound of Eq.\ref{ineq:sgd-bound},  the first term captures the impact of the training trajectory (``trajectory term''), and the second term captures the impact of the final solution, which in fact 
measures the flatness for the loss landscape at the found solution (``flatness term''). 
The learning rate in SGD has explicitly appeared in the trajectory term of Eq.\ref{eq:optimal-bound}.
From the way it appears in the bound, one may be tempted to assert that a small learning rate will improve generalization. This would then contradict some previous observations \citep{jastrzkebski2017three,wu2018sgd,HeLT19}, in which large learning rate will benefit generalization. In addition, the bound also suggests that the batch size has some direct impact on the trajectory term through gradient dispersion.  We investigate this by performing experiments with varying learning rates and  batch sizes (see Figure \ref{fig:lr-bs} in Appendix \ref{sec:lr-bs}).
As seen in the experiments, increasing the learning rate  impacts the trajectory term and the flatness term in opposite ways, i.e. making one increase and the other decrease. A similar (but reverted) behaviour is also observed with batch sizes.
This makes the generalization bound in Eq.\ref{eq:optimal-bound}, 
have a rather complex relationship with the settings of learning rate and batch size.

Eq.\ref{eq:optimal-bound} follows from Eq.\ref{ineq:sgd-bound} by minimizing the bound over $\sigma$. The bound in Eq.\ref{eq:optimal-bound} is thus independent of a choice of $\sigma_t$'s and can be computed easily and efficiently.

To summarize, we remark that these bounds suggest that in order for the model to generalize well, both the trajectory term and the flatness term need to be small --- the former involves the learning rate with the gradient dispersion along the training trajectory, whereas the latter depends on the flatness of the empirical risk surface at the found solution.

\section{Application in Linear Networks and Two-Layer ReLU Networks.} \label{sec:NNbound}

We now apply Theorem \ref{thm:re-neu-bound} 
to two neural network models 
in a regression setting. Let $Z=(X,Y)$ with $X\in\mathbb{R}^{d_0}$ and $Y\in\mathbb{R}$. Assume $||X||=1$. We use SGD to train a model $f(W,\cdot):\mathbb{R}^{d_0}\rightarrow\mathbb{R}$ and define the loss as $\ell(W,Z)=1/2(Y-f(W,X))^2$. The bounds below are both in the form of the product of the flatness term and the trajectory term.

\begin{thm}[Linear Networks] Let the model be a linear network, i.e. $f(W,X)=W^TX$, then
$
    \mathrm{gen}(\mu,P_{W_T|S})\leq
3\left(\sum_{t=1}^T\frac{R^2\lambda_t^2T}{4n} \ex{}{\ell(W_{t-1},Z)}\right)^{\frac{1}{3}}.
$
\label{thm:optimal-bound-linear}
\end{thm}



\begin{thm}[Two-Layer ReLU Networks]
Following \cite{arora2019fine}, consider $f(W,X)=\frac{1}{\sqrt{m}}\sum_{r=1}^m A_r \mathrm{ReLU}(W_r^TX)$ where $m$ is the width of the neural network,  $A_r\sim\mathrm{unif}(\{+1,-1\})$ and $\mathrm{ReLU}(\cdot)$ is the ReLU activation. We only train the first layer parameters $W=[W_1,W_2,\dots,W_m]\in\mathbb{R}^{d_0\times m}$ and fix the second layer parameters $A=[A_1,A_2,\dots,A_m]\in\mathbb{R}^m$ during training.  Then,
\[
    \mathrm{gen}(\mu,P_{W_T|S})\leq
3\left(\sum_{r=1}^m\ex{}{\frac{\mathbb{I}_{r,i,T}}{m}}\sum_{t=1}^T\frac{R^2\lambda_t^2T}{4n} \ex{}{\sum_{r=1}^m\frac{\mathbb{I}_{r,i,t}}{m}\ell(W_{t-1},Z)}\right)^{\frac{1}{3}},
\]
where $\mathbb{I}_{r,i,t}=\mathbb{I}\{W_{t-1,r}^TX_i\geq0\}$ and $\mathbb{I}$ is the indicator function.
\label{thm:optimal-bound-relu}
\end{thm}

Compared with Theorem \ref{thm:optimal-bound-linear}, we notice that in the two-layer ReLU network, the ReLU activation state along the training trajectory plays a key role in the bound of Theorem \ref{thm:optimal-bound-relu}. Specifically, the weights of the deactivated neurons do not contribute to the bound of Theorem \ref{thm:optimal-bound-relu},  making the bound not explicitly depend on the model dimension $d$. This result also suggests that sparsely activated ReLU networks are expected to generalize better.  
Despite various empirical evidence pointing to this behaviour (see, e.g., \cite{glorot2011deep}), to the best of our knowledge, this theorem provides the first theoretical justification in this regard.

The comparison between these bounds and the dynamics of generalization gap is deferred to Appendix \ref{appendix:sec-NN-bound}. In the remainder of this paper, we will
study more complex network architectures beyond the linear and two-layer ReLU networks and implications of our bounds in those settings.




\section{Experimental Study}
\label{sec:experiment}
\paragraph{Bound Verification} We first verify our bound of Eq.\ref{eq:optimal-bound} in Theorem \ref{thm:re-neu-bound} by training an MLP (with one hidden layer) and an AlexNet \citep{krizhevsky2012imagenet} on MNIST and CIFAR10 \citep{krizhevsky2009learning}, respectively. 
To simplify estimation, we fix the weight initialization and set $\sigma_t$ and $\lambda_t$ to be constants $\sigma$ and $\lambda$, respectively. 
To compute $\sum_{t=1}^T\ex{}{\mathbb{V}_t(W_{t-1})}$, 
we compute the gradient dispersion as its empirical estimate from a batch,
utilizing a PyTorch \citep{paszke2019pytorch} library BackPack \citep{DangelKH20}.
To compute $\mathrm{Tr}\left( \ex{}{\mathrm{H}_{W_T}(Z)}\right)$, we 
 use the PyHessian library \citep{yao2020pyhessian} 
 to compute the Hessian. 

We perform experiments with varying network width and varying levels of label noise. Specifically, noise level $\epsilon$ refers to the setting where we replace the labels of  $\epsilon$ fraction of the training and testing instances with random labels. The estimated bound is compared against the true generalization gap, namely, the difference between the training loss and testing loss, and is shown in Figure \ref{fig:bound-gap}.


\begin{figure}[ht!]
\begin{subfigure}{0.25\columnwidth}%
\centering%
\captionsetup{font=small}%
\scalebox{0.35}{
\begin{tikzpicture}
\begin{axis}
[legend style={nodes={scale=0.9, transform shape}},
ymode=log,
legend pos=north east,
xlabel=width,
table/col sep=comma,
grid style=dashed,
]
\addplot[mark=*, mark options={scale=1}, line width=2pt, color=blue, error bars/.cd, y dir=both, y explicit] table [y=gap1, x=wth, y error=g1_error]{data/Bound.csv};
\addlegendentry{gap}
\addplot[mark=*, mark options={scale=1}, line width=2pt, color=red, error bars/.cd, y dir=both, y explicit] table [y=b1, x=wth, y error=b1_error]{data/Bound.csv};
\addlegendentry{bound}
\label{plot-gap}
\end{axis}
\end{tikzpicture}
}
\caption{MLP on MNIST}%
\label{fig:gap-width}
\end{subfigure}%
\begin{subfigure}{0.25\columnwidth}%
\centering%
\captionsetup{font=small}%
\scalebox{0.35}{
\begin{tikzpicture}
\begin{axis}
[legend style={nodes={scale=0.9, transform shape}},
ymode=log,
legend pos=north east,
xlabel=filters,
table/col sep=comma,
grid style=dashed,
]
\addplot[mark=*, mark options={scale=1}, line width=2pt, color=blue, error bars/.cd, y dir=both, y explicit] table [y=gap3, x=flt, y error=g3_error]{data/Bound.csv};
\addlegendentry{gap}
\addplot[mark=*, mark options={scale=1}, line width=2pt, color=red, error bars/.cd, y dir=both, y explicit] table [y=b3, x=flt, y error=b3_error]{data/Bound.csv};
\addlegendentry{bound}
\label{plot-gap3}
\end{axis}
\end{tikzpicture}
}
\caption{AlexNet on CIFAR10}%
\label{fig:gap-width2}
\end{subfigure}%
\begin{subfigure}{0.25\columnwidth}%
\centering%
\captionsetup{font=small}%
\scalebox{0.35}{
\begin{tikzpicture}
\begin{axis}
[legend style={nodes={scale=0.9, transform shape}},
ymode=log,
legend pos=south east,
legend style={nodes={scale=1.25, transform shape}},
xlabel=label noise level,
table/col sep=comma,
grid style=dashed,
]
\addplot[mark=*, mark options={scale=1}, line width=2pt, color=blue, error bars/.cd, y dir=both, y explicit] table [y=gap2, x=noise2, y error=g2_error]{data/Bound.csv};
\addlegendentry{gap}
\addplot[mark=*, mark options={scale=1}, line width=2pt, color=red, error bars/.cd, y dir=both, y explicit] table [y=b2, x=noise2, y error=b2_error]{data/Bound.csv};
\addlegendentry{bound}
\label{plot-gap2}
\end{axis}
\end{tikzpicture}
}
\caption{MLP on MNIST}%
\label{fig:gap-noise}
\end{subfigure}%
\begin{subfigure}{0.25\columnwidth}%
\centering%
\captionsetup{font=small}%
\scalebox{0.35}{
\begin{tikzpicture}
\begin{axis}
[legend style={nodes={scale=0.9, transform shape}},
ymode=log,
legend pos=south east,
xlabel=label noise level,
table/col sep=comma,
grid style=dashed,
]
\addplot[mark=*, mark options={scale=1}, line width=2pt, color=blue, error bars/.cd, y dir=both, y explicit] table [y=gap4, x=noise4, y error=g4_error]{data/Bound.csv};
\addlegendentry{gap}
\addplot[mark=*, mark options={scale=1}, line width=2pt, color=red, error bars/.cd, y dir=both, y explicit] table [y=b4, x=noise4, y error=b4_error]{data/Bound.csv};
\addlegendentry{bound}
\label{plot-gap4}
\end{axis}
\end{tikzpicture}
}
\caption{AlexNet on CIFAR10}%
\label{fig:gap-noise2}
\end{subfigure}%
\caption{Estimated bound and empirical generalization gap (``gap'') as functions of network width ((a) and (b)) and label noise level ((c) and (d)). Note that Y-axis is in log scale.
}
\label{fig:bound-gap}
\end{figure}
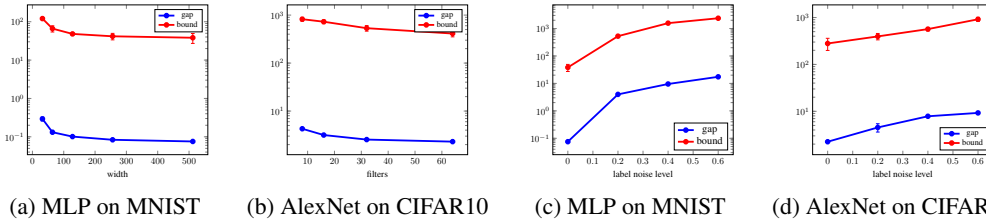

In Figure \ref{fig:bound-gap}, we see that in all cases the estimated bound follows closely the trend of the true generalization gap.  The fact that the bound curve consistently tracks the gap curve under various label noise levels indicates that our bound very well captures the changes of the data distribution.  Note that in Figure \ref{fig:bound-gap} (a) and (b), our bound decays with the increase of the model size, showing a trend as opposite to the bounds obtained in classical learning theory. But such a trend clearly better explains the generalization behaviour of modern neural networks.

\paragraph{{Epoch-wise Double Descent of Gradient Dispersion}}
\label{sec:double descent}
We experimentally investigate the impact of gradient dispersion on the training of the neural networks by
fixing the learning rate, batch size and weight initialization for the each model (MLP for MNIST, AlexNet for CIFAR10). For each model and various 
label noise levels, we plot in Figure \ref{fig:acc-var} the 
evolution of the (empirical) gradient dispersion $\widehat{\mathbb{V}}_t(w_{t-1})$, training accuracy and testing accuracy across training epochs. 

An intriguing epoch-wise ``double descent'' phenomenon is observed, particularly when the labels are noisy.
According to the double descent curve, the training may be split into three phases (e.g., Figure \ref{fig:acc-var} (h)). 
In the first phase, the gradient dispersion rapidly descends and maintains a very low level. 
In this phase, both training and test accuracies increase while maintaining a very small generalization gap. This suggests that the network in this phase is extracting useful patterns and generalizes well. In the second phase, the gradient dispersion starts increasing until it reaches a peak value. In this phase, the training and testing accuracies gradually diverge, marking the model entering an overfitting or ``memorization'' regime -- when the data contains the noisy labels, the network mostly tries to memorize the labels in the training set. In the third phase, the gradient dispersion descends again, reaching a low value. In this phase, the model continuously overfits the training data, until the training and testing curves  reach their respective maximum and minimum. It appears that the timing of the three phases depends on the dataset and the label noise level. For simpler data (e.g. MNIST) and cleaner datasets 
, the first phase may be shorter. This is arguably because in these datasets, extracting useful patterns is relatively easier. Nonetheless, the valley in the double-descent curve appears to mark a ``great divide'' between generalization and memorization. 

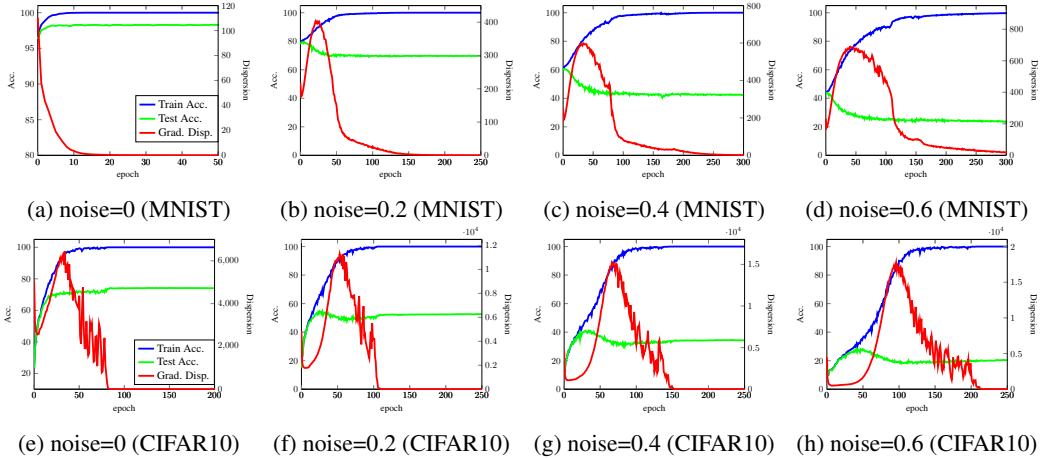
\begin{figure}[ht!]
\begin{subfigure}[b]{0.25\columnwidth}%
\centering%
\captionsetup{font=small}%
\scalebox{0.35}{
\begin{tikzpicture}
\begin{axis}
[
axis y line*=left,
ymin=80, ymax=101,
xlabel=epoch,
xmin=0,
    xmax=50,
ylabel=Acc. 
]
\addplot[line width=2pt, color=blue] table [y=TrAcc0, x=epoch0]{data/MLP.txt};
\label{plot-tracc-mlp}
\addplot[line width=2pt, color=green] table [y=TsAcc0, x=epoch0]{data/MLP.txt};
\label{plot-tsacc-mlp}
\end{axis}
\begin{axis}
[
legend style={nodes={scale=1.25, transform shape},at={(1,0.4)}},
legend cell align={left},
axis y line*=right,
 ymin=0, ymax=120,
  ylabel=Dispersion,
  ylabel style={rotate=-180},
  xmin=0,
    xmax=50,
]
\addlegendimage{/pgfplots/refstyle=plot-tracc-mlp}\addlegendentry{Train Acc.}
\addlegendimage{/pgfplots/refstyle=plot-tsacc-mlp}\addlegendentry{Test Acc.}
\addplot[line width=2pt, color=red] table [y=GVar0, x=epoch0]{data/MLP.txt};
\addlegendentry{Grad. Disp.}
\end{axis}
\end{tikzpicture}
}
\caption{noise=0 (MNIST)}%
\label{fig:acc-var-mlp}
\end{subfigure}%
\begin{subfigure}[b]{0.25\columnwidth}%
\centering%
\captionsetup{font=small}%
\scalebox{0.35}{
\begin{tikzpicture}
\begin{axis}
[
axis y line*=left,
ymin=0, ymax=105,
xlabel=epoch,
grid style=dashed,
xmin=0,
    xmax=250,
ylabel=Acc., 
]
\addplot[line width=2pt, color=blue] table [y=TrAcc1, x=epoch0]{data/MLP.txt};
\addplot[line width=2pt, color=green] table [y=TsAcc1, x=epoch0]{data/MLP.txt};
\end{axis}
\begin{axis}
[legend pos=south east,
axis y line*=right,
 ymin=0, ymax=450,
grid style=dashed,
  ylabel=Dispersion,
  ylabel style={rotate=-180},
  xmin=0,
    xmax=250,
]
\addplot[line width=2pt, color=red] table [y=GVar1, x=epoch0]{data/MLP.txt};
\end{axis}
\end{tikzpicture}
}
\caption{noise=0.2 (MNIST)}%
\label{fig:acc-var-mlp02}
\end{subfigure}%
\begin{subfigure}[b]{0.25\columnwidth}%
\centering%
\captionsetup{font=small}%
\scalebox{0.35}{
\begin{tikzpicture}
\begin{axis}
[
axis y line*=left,
ymin=0, ymax=105,
xlabel=epoch,
grid style=dashed,
xmin=0,
    xmax=300,
ylabel=Acc., 
]
\addplot[line width=2pt, color=blue] table [y=TrAcc2, x=epoch0]{data/MLP.txt};
\addplot[line width=2pt, color=green] table [y=TsAcc2, x=epoch0]{data/MLP.txt};
\end{axis}
\begin{axis}
[
axis y line*=right,
 ymin=0, ymax=800,
grid style=dashed,
  ylabel=Dispersion,
  ylabel style={rotate=-180},
  xmin=0,
    xmax=300,
]
\addplot[line width=2pt, color=red] table [y=GVar2, x=epoch0]{data/MLP.txt};
\end{axis}
\end{tikzpicture}
}
\caption{noise=0.4 (MNIST)}%
\label{fig:acc-var-mlp04}
\end{subfigure}%
\begin{subfigure}[b]{0.25\columnwidth}%
\centering%
\captionsetup{font=small}%
\scalebox{0.35}{
\begin{tikzpicture}
\begin{axis}
[
axis y line*=left,
ymin=0, ymax=105,
xlabel=epoch,
grid style=dashed,
xmin=0,
    xmax=300,
ylabel=Acc. 
]
\addplot[line width=2pt, color=blue] table [y=TrAcc3, x=epoch0]{data/MLP.txt};
\addplot[line width=2pt, color=green] table [y=TsAcc3, x=epoch0]{data/MLP.txt};
\end{axis}
\begin{axis}
[legend pos=south east,
axis y line*=right,
 ymin=0, ymax=950,
grid style=dashed,
  ylabel=Dispersion,
  ylabel style={rotate=-180},
  xmin=0,
    xmax=300,
]
\addplot[line width=2pt, color=red] table [y=GVar3, x=epoch0]{data/MLP.txt};
\end{axis}
\end{tikzpicture}
}
\caption{noise=0.6 (MNIST)}%
\label{fig:acc-var-mlp06}
\end{subfigure}%

\begin{subfigure}[b]{0.25\columnwidth}%
\centering%
\captionsetup{font=small}%
\scalebox{0.35}{
\begin{tikzpicture}
\begin{axis}
[
axis y line*=left,
ymin=10, ymax=105,
xlabel=epoch,
grid style=dashed,
xmin=0,
    xmax=200,
ylabel=Acc. 
]
\addplot[line width=2pt, color=blue] table [y=TrAcc0, x=epoch0]{data/Alex.txt};
\label{plot-tracc-alex}
\addplot[line width=2pt, color=green] table [y=TsAcc0, x=epoch0]{data/Alex.txt};
\label{plot-tsacc-alex}
\end{axis}
\begin{axis}
[
legend style={nodes={scale=1.2, transform shape},at={(1,0.33)}},
legend cell align={left},
axis y line*=right,
 ymin=0, ymax=7000,
grid style=dashed,
  ylabel=Dispersion,
  ylabel style={rotate=-180},
  xmin=0,
    xmax=200,
]
\addlegendimage{/pgfplots/refstyle=plot-tracc-alex}\addlegendentry{Train Acc.}
\addlegendimage{/pgfplots/refstyle=plot-tsacc-alex}\addlegendentry{Test Acc.}
\addplot[line width=2pt, color=red] table [y=GVar0, x=epoch0]{data/Alex.txt};
\addlegendentry{Grad. Disp.}
\end{axis}
\end{tikzpicture}
}
\caption{noise=0 (CIFAR10)}%
\label{fig:acc-var-alex}
\end{subfigure}%
\begin{subfigure}[b]{0.25\columnwidth}%
\centering%
\captionsetup{font=small}%
\scalebox{0.35}{
\begin{tikzpicture}
\begin{axis}
[
axis y line*=left,
ymin=0, ymax=105,
xlabel=epoch,
grid style=dashed,
xmin=0,
    xmax=250,
ylabel=Acc. 
]
\addplot[line width=2pt, color=blue] table [y=TrAcc1, x=epoch0]{data/Alex.txt};
\addplot[line width=2pt, color=green] table [y=TsAcc1, x=epoch0]{data/Alex.txt};
\end{axis}
\begin{axis}
[legend pos=south east,
axis y line*=right,
 ymin=0, ymax=12500,
grid style=dashed,
  ylabel=Dispersion,
  ylabel style={rotate=-180},
  xmin=0,
    xmax=250,
]
\addplot[line width=2pt, color=red] table [y=GVar1, x=epoch0]{data/Alex.txt};
\end{axis}
\end{tikzpicture}
}
\caption{noise=0.2 (CIFAR10)}%
\label{fig:acc-var-alex02}
\end{subfigure}%
\begin{subfigure}[b]{0.25\columnwidth}%
\centering%
\captionsetup{font=small}%
\scalebox{0.35}{
\begin{tikzpicture}
\begin{axis}
[
axis y line*=left,
ymin=0, ymax=105,
xlabel=epoch,
grid style=dashed,
xmin=0,
    xmax=250,
ylabel=Acc. 
]
\addplot[line width=2pt, color=blue] table [y=TrAcc2, x=epoch0]{data/Alex.txt};
\addplot[line width=2pt, color=green] table [y=TsAcc2, x=epoch0]{data/Alex.txt};
\end{axis}
\begin{axis}
[legend pos=south east,
axis y line*=right,
 ymin=0, ymax=18000,
grid style=dashed,
  ylabel=Dispersion,
  ylabel style={rotate=-180},
  xmin=0,
    xmax=250,
]
\addplot[line width=2pt, color=red] table [y=GVar2, x=epoch0]{data/Alex.txt};
\end{axis}
\end{tikzpicture}
}
\caption{noise=0.4 (CIFAR10)}%
\label{fig:acc-var-alex04}
\end{subfigure}%
\begin{subfigure}[b]{0.25\columnwidth}%
\centering%
\captionsetup{font=small}%
\scalebox{0.35}{
\begin{tikzpicture}
\begin{axis}
[
axis y line*=left,
ymin=0, ymax=105,
xlabel=epoch,
grid style=dashed,
xmin=0,
    xmax=250,
ylabel=Acc. 
]
\addplot[line width=2pt, color=blue] table [y=TrAcc3, x=epoch0]{data/Alex.txt};
\addplot[line width=2pt, color=green] table [y=TsAcc3, x=epoch0]{data/Alex.txt};
\end{axis}
\begin{axis}
[legend pos=south east,
axis y line*=right,
 ymin=0, ymax=21000,
grid style=dashed,
  ylabel=Dispersion,
  ylabel style={rotate=-180},
  xmin=0,
    xmax=250,
]
\addplot[line width=2pt, color=red] table [y=GVar3, x=epoch0]{data/Alex.txt};
\end{axis}
\end{tikzpicture}
}
\caption{noise=0.6 (CIFAR10)}%
\label{fig:acc-var-alex06}
\end{subfigure}%
\caption{Dynamics of gradient dispersion, in relation to training/testing accuracies.
}
\label{fig:acc-var}
\end{figure}

\paragraph{Dynamic Gradient Clipping}
Inspired by our generalization bounds and above observations, one way to reduce the generalization error is to control the trajectory term of the bounds by reducing the gradient dispersion in each training step. 
Here we investigate a simple scheme that dynamically clips the gradient norm so as to reduce the gradient dispersion.
Specifically, 
whenever the current gradient norm is larger than the gradient norm $K$ steps earlier, or  
$||g(W_{t},B_t)||_2>||g(W_{t-K},B_{t-K})||_2$ (i.e., the model is expected to have entered the ``memorization'' regime), we reduce the norm of the current gradient $g(W_{t},B_t)$ to $\alpha$ fraction of $||g(W_{t-K},B_{t-K})||_2$, for some prescribed value $\alpha<1$. 

\begin{wrapfigure}{r}{0.48\textwidth}
\begin{subfigure}[b]{0.23\columnwidth}%
\centering%
\captionsetup{font=small}%
\scalebox{0.35}{
\begin{tikzpicture}
\begin{axis}
[legend style={nodes={scale=1.25, transform shape}},
legend cell align={left},
legend pos=south east,
ymin=0, ymax=105,
xlabel=epoch,
grid style=dashed,
xmin=0,
    xmax=200,
ylabel=Acc.
]
\addplot[line width=2pt, color=black!30!green] table [y=TrAcc1, x=epoch0]{data/MLP_old.txt};
\addlegendentry{Train Acc. w/o clipping}
\addplot[line width=2pt, color={orange}] table [y=TsAcc1, x=epoch0]{data/MLP_old.txt};
\addlegendentry{Test Acc. w/o clipping}
\addplot[line width=2pt, color=red] table [y=TrAcc, x=epoch]{data/MLP02dc.txt};
\addlegendentry{Train Acc. w/ clipping}
\addplot[line width=2pt, color={blue}] table [y=TsAcc, x=epoch]{data/MLP02dc.txt};
\addlegendentry{Test Acc. w/ clipping}
\end{axis}
\end{tikzpicture}
}
\caption{noise=0.2 (MNIST)}%
\end{subfigure}
\begin{subfigure}[b]{0.23\columnwidth}%
\centering%
\captionsetup{font=small}%
\scalebox{0.35}{
\begin{tikzpicture}
\begin{axis}
[
legend cell align={left},
legend pos=south east,
ymin=0, ymax=105,
xlabel=epoch,
grid style=dashed,
xmin=0,
    xmax=200,
ylabel=Acc.
]
\addplot[line width=2pt, color=black!30!green] table [y=TrAcc2, x=epoch0]{data/MLP_old.txt};
\addplot[line width=2pt, color={orange}] table [y=TsAcc2, x=epoch0]{data/MLP_old.txt};
\addplot[line width=2pt, color=red] table [y=TrAcc, x=epoch]{data/MLP04dc.txt};
\addplot[line width=2pt, color=blue] table [y=TsAcc, x=epoch]{data/MLP04dc.txt};
\end{axis}
\end{tikzpicture}
}
\caption{noise=0.4 (MNIST)}%
\end{subfigure}
\caption{Dynamic Gradient Clipping.}
\label{fig:DyGC}
\end{wrapfigure}
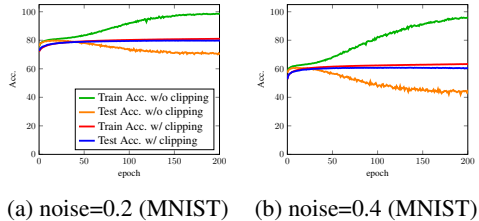

The effectiveness of this scheme is best demonstrated when the labels contain noise. As shown in
Figure \ref{fig:DyGC} (and Figure \ref{fig:DyGC-alex} in Appendix \ref{sec:DYG-appendix}), dynamic gradient clipping significantly closes the gap between the training accuracy and the testing accuracy.  The models trained with this scheme maintain a near-optimal testing accuracy (e.g., about 80\% when the label noise level of MNIST is 0.2), without suffering from the severe memorization effect as seen in models trained without this scheme. Further understanding of the double-descent phenomenon of the gradient dispersion may enable more delicate design of such a dynamic clipping scheme  and potentially lead to  novel and powerful regularization techniques.



\section{A Practical Implication: Gaussian Model Perturbation}
\label{sec:regularization}

The appearance of the flatness term in our generalization bounds suggests that for an empirical risk minimizer $w^*$ to generalize well, it is necessary that the empirical risk surface at $w^*$ is flat, or insensitive to a small perturbation of $w^*$. This naturally motivates a training scheme using the following regularized loss:
\[
\min_{w} L_s(w) + \rho\ex{\Delta\sim\mathcal{N}(0,\sigma^2\mathbf{I}_d)}{L_s(w+\Delta)-L_s(w)},
\]
where $\rho$ is a hyper-parameter. 
Replacing the expectation above with its stochastic approximation using $k$ realizations of $\Delta$ gives rise to the following optimization problem.
\[
\min_{w} \frac{1}{b}\sum_{z\in B}\left(
(1-\rho)
\ell(w,z)+\rho\frac{1}{k}\sum_{i=1}^k\left(\ell(w+\delta_i,z)\right)\right).
\]

We refer to the SGD training scheme using this loss as
{\em Gaussian model perturbation} or 
GMP. 
Notably, GMP requires $k+1$ forward passes for every parameter update. 
Empirical evidence shows that a small $k$, for example, $k=3$, already gives competitive performance. Implementing the $k+1$ forward passes on parallel processors further reduces the computation load.

\begin{wraptable}{l}{0.6\textwidth}
\begin{tabular}{l|c|c|c}
\toprule
Method & SVHN & CIFAR-10 & CIFAR-100\\
\midrule
ERM & 96.86$\pm$0.060 & 93.68$\pm$0.193 & 72.16$\pm$0.297 \\
Dropout & 97.04$\pm$0.049 & 93.78$\pm$0.147 & 72.28$\pm$0.337  \\
L. S. & 96.93$\pm$0.070 & 93.71$\pm$0.158 & 72.51$\pm$0.179 \\
Flooding & 96.85$\pm$0.085 & 93.74$\pm$0.145 & 72.07$\pm$0.271 \\
MixUp & 96.91$\pm$0.057  & \textbf{94.52$\pm$0.112} & 73.19$\pm$0.254 \\
Adv. Tr. & 97.06$\pm$0.091 & 93.51$\pm$0.130 & 70.88$\pm$0.145 \\
AMP & \textbf{97.27$\pm$0.015}  & {94.35$\pm$0.147} & {74.40$\pm$0.168} \\
\textbf{GMP$^3$} & \underline{97.18$\pm$0.057} & 94.33$\pm$0.094 & \underline{74.45$\pm$0.256} \\
\textbf{GMP$^{10}$} & {97.09$\pm$0.068} & \underline{94.45$\pm$0.158}  & \textbf{75.09$\pm$0.285}  \\
\bottomrule
\end{tabular}
\caption{
Top-1 classification accuracy acc.($\%$) of VGG16. 
We run experiments 10 times and report the mean and the standard deviation of the testing accuracy. 
Superscript denotes the number of sampled Gaussian noises during training.
}
\label{tab:cvresults}
\end{wraptable}

We compare GMP with several major regularization 
schemes in the current art, including
Dropout \citep{srivastava2014dropout}, label smoothing (L.S.) \citep{szegedy2016rethinking}, Flooding \citep{ishida2020we}, MixUp \citep{zhang2018mixup}, adversarial training (Adv. Tr.) \citep{goodfellow2015explaining}, and AMP \citep{zheng2020regularizing}. 
We will also include ERM as the baseline, where no regularization method applied.
The compared schemes are evaluated on three popular benchmark image classification datasets SVHN \citep{netzer2011reading}, CIFAR-10 and CIFAR-100 \citep{krizhevsky2009learning}.
Two representative deep architectures VGG16 \citep{simonyan2015very} and PreActResNet18 \citep{he2016identity} are taken as the underlying model. We train the models for $200$ epochs by SGD.
The learning rate is initialized as $0.1$ and divided by $10$ after $100$ and $150$ epochs. For all compared models, the batch size is set to $50$ and weight decay is set to $10^{-4}$. 
For GMP, we choose $\rho=0.5$ and set the standard deviation of the Gaussian noise $\Delta$ to $0.03$. The value of $k$ is chosen as $3$ and $10$ respectively (referred to as ${\rm GMP}^3$ and 
${\rm GMP}^{10}$).

The performances of all compared schemes are given in Table \ref{tab:cvresults} (and the results of PreActResNet18 are shown in Appendix \ref{sec:gmp-appendix}).
For the compared regularization schemes except GMP, we directly report their performances as given in \cite{zheng2020regularizing}. 
Table \ref{tab:cvresults} demonstrates the effectiveness of GMP. Overall GMP performs comparably to the current art of regularization schemes, although appearing slighly inferior to 
the most recent record given by 
AMP on SVHN and MixUp on CIFAR-10, respectively 
\citep{zheng2020regularizing}.  Noting that the key ingredient of AMP,  ``max-pooling'' in the parameter space, greatly resembles regularization term in GMP, which may be seen as ``average-pooling'' in the same space.

\section{Conclusion and Outlook}

This paper presents new information theoretic generalization bounds for models (e.g., linear networks and two-layer ReLU neural networks) trained with SGD. Our bounds naturally point to new and effective regularization schemes. At the same time, our bounds and experimental study reveal interesting phenomena in the SGD training of neural networks. 
There are yet promising directions for further improving these bounds, for example, via exploiting conditional mutual information bounds (\cite{haghifam2020sharpened}), strong data processing inequalities (\cite{wang2021analyzing}), and the relationship between between the trajectory term and the flatness term. 

\subsubsection*{Acknowledgments}
This work is supported partly by an NSERC Discovery Grant and Artificial Intelligence for Design Challenge program of National Research Council Canada. The authors would like to sincerely thank Gergely Neu and Borja Rodr\'iguez-G\'alvez for pointing out errors in the previous version of this paper. The authors would also like to thank Maia Fraser for stimulating discussions.

\bibliographystyle{iclr2022_conference}
\bibliography{ref}

\begin{thebibliography}{84}
\providecommand{\natexlab}[1]{#1}
\providecommand{\url}[1]{\texttt{#1}}
\expandafter\ifx\csname urlstyle\endcsname\relax
  \providecommand{\doi}[1]{doi: #1}\else
  \providecommand{\doi}{doi: \begingroup \urlstyle{rm}\Url}\fi

\bibitem[Allen-Zhu et~al.(2019)Allen-Zhu, Li, and Song]{allen2019convergence}
Zeyuan Allen-Zhu, Yuanzhi Li, and Zhao Song.
\newblock A convergence theory for deep learning via over-parameterization.
\newblock In \emph{International Conference on Machine Learning}, pp.\
  242--252. PMLR, 2019.

\bibitem[Arora et~al.(2018)Arora, Ge, Neyshabur, and Zhang]{arora2018stronger}
Sanjeev Arora, Rong Ge, Behnam Neyshabur, and Yi~Zhang.
\newblock Stronger generalization bounds for deep nets via a compression
  approach.
\newblock In \emph{International Conference on Machine Learning}, pp.\
  254--263. PMLR, 2018.

\bibitem[Arora et~al.(2019)Arora, Du, Hu, Li, and Wang]{arora2019fine}
Sanjeev Arora, Simon Du, Wei Hu, Zhiyuan Li, and Ruosong Wang.
\newblock Fine-grained analysis of optimization and generalization for
  overparameterized two-layer neural networks.
\newblock In \emph{International Conference on Machine Learning}, pp.\
  322--332. PMLR, 2019.

\bibitem[Asadi \& Abbe(2020)Asadi and Abbe]{asadi2020chaining}
Amir~R Asadi and Emmanuel Abbe.
\newblock Chaining meets chain rule: Multilevel entropic regularization and
  training of neural networks.
\newblock \emph{Journal of Machine Learning Research}, 21\penalty0
  (139):\penalty0 1--32, 2020.

\bibitem[Asadi et~al.(2018)Asadi, Abbe, and Verd{\'u}]{asadi2018chaining}
Amir~R Asadi, Emmanuel Abbe, and Sergio Verd{\'u}.
\newblock Chaining mutual information and tightening generalization bounds.
\newblock In \emph{Proceedings of the 32nd International Conference on Neural
  Information Processing Systems}, pp.\  7245--7254, 2018.

\bibitem[Bartlett et~al.(2017)Bartlett, Foster, and
  Telgarsky]{bartlett2017spectrally}
Peter~L Bartlett, Dylan~J Foster, and Matus Telgarsky.
\newblock Spectrally-normalized margin bounds for neural networks.
\newblock In \emph{Proceedings of the 31st International Conference on Neural
  Information Processing Systems}, pp.\  6241--6250, 2017.

\bibitem[Bassily et~al.(2018)Bassily, Moran, Nachum, Shafer, and
  Yehudayoff]{bassily2018learners}
Raef Bassily, Shay Moran, Ido Nachum, Jonathan Shafer, and Amir Yehudayoff.
\newblock Learners that use little information.
\newblock In \emph{Algorithmic Learning Theory}, pp.\  25--55. PMLR, 2018.

\bibitem[Bassily et~al.(2020)Bassily, Feldman, Guzm{\'a}n, and
  Talwar]{bassily2020stability}
Raef Bassily, Vitaly Feldman, Crist{\'o}bal Guzm{\'a}n, and Kunal Talwar.
\newblock Stability of stochastic gradient descent on nonsmooth convex losses.
\newblock \emph{Advances in Neural Information Processing Systems}, 33, 2020.

\bibitem[Belkin et~al.(2019)Belkin, Hsu, Ma, and Mandal]{belkin2019reconciling}
Mikhail Belkin, Daniel Hsu, Siyuan Ma, and Soumik Mandal.
\newblock Reconciling modern machine-learning practice and the classical
  bias--variance trade-off.
\newblock \emph{Proceedings of the National Academy of Sciences}, 116\penalty0
  (32):\penalty0 15849--15854, 2019.

\bibitem[Bishop(1995)]{bishop1995training}
Chris~M Bishop.
\newblock Training with noise is equivalent to tikhonov regularization.
\newblock \emph{Neural computation}, 7\penalty0 (1):\penalty0 108--116, 1995.

\bibitem[Bottou et~al.(2018)Bottou, Curtis, and
  Nocedal]{bottou2018optimization}
L{\'e}on Bottou, Frank~E Curtis, and Jorge Nocedal.
\newblock Optimization methods for large-scale machine learning.
\newblock \emph{Siam Review}, 60\penalty0 (2):\penalty0 223--311, 2018.

\bibitem[Bousquet \& Elisseeff(2002)Bousquet and
  Elisseeff]{bousquet2002stability}
Olivier Bousquet and Andr{\'e} Elisseeff.
\newblock Stability and generalization.
\newblock \emph{The Journal of Machine Learning Research}, 2:\penalty0
  499--526, 2002.

\bibitem[Bu et~al.(2020)Bu, Zou, and Veeravalli]{bu2020tightening}
Yuheng Bu, Shaofeng Zou, and Venugopal~V Veeravalli.
\newblock Tightening mutual information-based bounds on generalization error.
\newblock \emph{IEEE Journal on Selected Areas in Information Theory},
  1\penalty0 (1):\penalty0 121--130, 2020.

\bibitem[Camuto et~al.(2020)Camuto, Willetts, Simsekli, Roberts, and
  Holmes]{CamutoWSRH20}
Alexander Camuto, Matthew Willetts, Umut Simsekli, Stephen~J. Roberts, and
  Chris~C. Holmes.
\newblock Explicit regularisation in gaussian noise injections.
\newblock In \emph{Advances in Neural Information Processing Systems}, 2020.

\bibitem[Camuto et~al.(2021)Camuto, Wang, Zhu, Holmes, G{\"u}rb{\"u}zbalaban,
  and {\c{S}}im{\c{s}}ekli]{camuto2021asymmetric}
Alexander Camuto, Xiaoyu Wang, Lingjiong Zhu, Chris Holmes, Mert
  G{\"u}rb{\"u}zbalaban, and Umut {\c{S}}im{\c{s}}ekli.
\newblock Asymmetric heavy tails and implicit bias in gaussian noise
  injections.
\newblock \emph{arXiv preprint arXiv:2102.07006}, 2021.

\bibitem[Chaudhari et~al.(2017)Chaudhari, Choromanska, Soatto, LeCun, Baldassi,
  Borgs, Chayes, Sagun, and Zecchina]{ChaudhariCSLBBC17}
Pratik Chaudhari, Anna Choromanska, Stefano Soatto, Yann LeCun, Carlo Baldassi,
  Christian Borgs, Jennifer~T. Chayes, Levent Sagun, and Riccardo Zecchina.
\newblock Entropy-sgd: Biasing gradient descent into wide valleys.
\newblock In \emph{5th International Conference on Learning Representations}.
  OpenReview.net, 2017.

\bibitem[Chen et~al.(2018)Chen, Jin, and Yu]{chen2018stability}
Yuansi Chen, Chi Jin, and Bin Yu.
\newblock Stability and convergence trade-off of iterative optimization
  algorithms.
\newblock \emph{arXiv preprint arXiv:1804.01619}, 2018.

\bibitem[Cover \& Thomas(2012)Cover and Thomas]{cover2012elements}
Thomas~M Cover and Joy~A Thomas.
\newblock \emph{Elements of Information Theory}.
\newblock John Wiley \& Sons, 2012.

\bibitem[Dangel et~al.(2020)Dangel, Kunstner, and Hennig]{DangelKH20}
Felix Dangel, Frederik Kunstner, and Philipp Hennig.
\newblock Backpack: Packing more into backprop.
\newblock In \emph{8th International Conference on Learning Representations}.
  OpenReview.net, 2020.

\bibitem[Dinh et~al.(2017)Dinh, Pascanu, Bengio, and Bengio]{dinh2017sharp}
Laurent Dinh, Razvan Pascanu, Samy Bengio, and Yoshua Bengio.
\newblock Sharp minima can generalize for deep nets.
\newblock In \emph{International Conference on Machine Learning}, pp.\
  1019--1028. PMLR, 2017.

\bibitem[Dwork et~al.(2015)Dwork, Feldman, Hardt, Pitassi, Reingold, and
  Roth]{dwork2015generalization}
Cynthia Dwork, Vitaly Feldman, Moritz Hardt, Toniann Pitassi, Omer Reingold,
  and Aaron Roth.
\newblock Generalization in adaptive data analysis and holdout reuse.
\newblock In \emph{Proceedings of the 28th International Conference on Neural
  Information Processing Systems-Volume 2}, pp.\  2350--2358, 2015.

\bibitem[Dziugaite \& Roy(2017)Dziugaite and Roy]{DziugaiteR17}
Gintare~Karolina Dziugaite and Daniel~M. Roy.
\newblock Computing nonvacuous generalization bounds for deep (stochastic)
  neural networks with many more parameters than training data.
\newblock In \emph{Proceedings of the 33rd Annual Conference on Uncertainty in
  Artificial Intelligence (UAI)}, 2017.

\bibitem[Faghri et~al.(2020)Faghri, Duvenaud, Fleet, and Ba]{faghri2020study}
Fartash Faghri, David Duvenaud, David~J Fleet, and Jimmy Ba.
\newblock A study of gradient variance in deep learning.
\newblock \emph{arXiv preprint arXiv:2007.04532}, 2020.

\bibitem[Feldman \& Vondrak(2019)Feldman and Vondrak]{feldman2019high}
Vitaly Feldman and Jan Vondrak.
\newblock High probability generalization bounds for uniformly stable
  algorithms with nearly optimal rate.
\newblock In \emph{Conference on Learning Theory}, pp.\  1270--1279. PMLR,
  2019.

\bibitem[Foret et~al.(2020)Foret, Kleiner, Mobahi, and
  Neyshabur]{foret2020sharpness}
Pierre Foret, Ariel Kleiner, Hossein Mobahi, and Behnam Neyshabur.
\newblock Sharpness-aware minimization for efficiently improving
  generalization.
\newblock \emph{arXiv preprint arXiv:2010.01412}, 2020.

\bibitem[Gelfand \& Mitter(1991)Gelfand and Mitter]{gelfand1991recursive}
Saul~B Gelfand and Sanjoy~K Mitter.
\newblock Recursive stochastic algorithms for global optimization in r\^{}d.
\newblock \emph{SIAM Journal on Control and Optimization}, 29\penalty0
  (5):\penalty0 999--1018, 1991.

\bibitem[Glorot et~al.(2011)Glorot, Bordes, and Bengio]{glorot2011deep}
Xavier Glorot, Antoine Bordes, and Yoshua Bengio.
\newblock Deep sparse rectifier neural networks.
\newblock In \emph{14th International Conference on Artificial Intelligence and
  Statistics}, volume~15, pp.\  315--323, 2011.

\bibitem[Goodfellow et~al.(2015)Goodfellow, Shlens, and
  Szegedy]{goodfellow2015explaining}
Ian~J. Goodfellow, Jonathon Shlens, and Christian Szegedy.
\newblock Explaining and harnessing adversarial examples.
\newblock In \emph{3rd International Conference on Learning Representations,
  {ICLR}}, 2015.

\bibitem[Guo et~al.(2005)Guo, Shamai, and Verd{\'u}]{guo2005mutual}
Dongning Guo, Shlomo Shamai, and Sergio Verd{\'u}.
\newblock Mutual information and minimum mean-square error in gaussian
  channels.
\newblock \emph{IEEE transactions on information theory}, 51\penalty0
  (4):\penalty0 1261--1282, 2005.

\bibitem[Hafez-Kolahi et~al.(2020)Hafez-Kolahi, Golgooni, Kasaei, and
  Soleymani]{hafez2020conditioning}
Hassan Hafez-Kolahi, Zeinab Golgooni, Shohreh Kasaei, and Mahdieh Soleymani.
\newblock Conditioning and processing: Techniques to improve
  information-theoretic generalization bounds.
\newblock \emph{Advances in Neural Information Processing Systems}, 33, 2020.

\bibitem[Haghifam et~al.(2020)Haghifam, Negrea, Khisti, Roy, and
  Dziugaite]{haghifam2020sharpened}
Mahdi Haghifam, Jeffrey Negrea, Ashish Khisti, Daniel~M Roy, and
  Gintare~Karolina Dziugaite.
\newblock Sharpened generalization bounds based on conditional mutual
  information and an application to noisy, iterative algorithms.
\newblock \emph{arXiv preprint arXiv:2004.12983}, 2020.

\bibitem[Hardt et~al.(2016)Hardt, Recht, and Singer]{hardt2016train}
Moritz Hardt, Ben Recht, and Yoram Singer.
\newblock Train faster, generalize better: Stability of stochastic gradient
  descent.
\newblock In \emph{International Conference on Machine Learning}, pp.\
  1225--1234. PMLR, 2016.

\bibitem[He et~al.(2019)He, Liu, and Tao]{HeLT19}
Fengxiang He, Tongliang Liu, and Dacheng Tao.
\newblock Control batch size and learning rate to generalize well: Theoretical
  and empirical evidence.
\newblock In \emph{Advances in Neural Information Processing Systems 32}, pp.\
  1141--1150, 2019.

\bibitem[He et~al.(2016)He, Zhang, Ren, and Sun]{he2016identity}
Kaiming He, Xiangyu Zhang, Shaoqing Ren, and Jian Sun.
\newblock Identity mappings in deep residual networks.
\newblock In \emph{European Conference on Computer Vision}, pp.\  630--645,
  2016.

\bibitem[Hochreiter \& Schmidhuber(1997)Hochreiter and
  Schmidhuber]{hochreiter1997flat}
Sepp Hochreiter and J{\"u}rgen Schmidhuber.
\newblock Flat minima.
\newblock \emph{Neural Computation}, 9\penalty0 (1):\penalty0 1--42, 1997.

\bibitem[Ishida et~al.(2020)Ishida, Yamane, Sakai, Niu, and
  Sugiyama]{ishida2020we}
Takashi Ishida, Ikko Yamane, Tomoya Sakai, Gang Niu, and Masashi Sugiyama.
\newblock Do we need zero training loss after achieving zero training error?
\newblock In \emph{Proceedings of the 37th International Conference on Machine
  Learning, {ICML}}, 2020.

\bibitem[Jastrzebski et~al.(2017)Jastrzebski, Kenton, Arpit, Ballas, Fischer,
  Bengio, and Storkey]{jastrzkebski2017three}
Stanislaw Jastrzebski, Zachary Kenton, Devansh Arpit, Nicolas Ballas, Asja
  Fischer, Yoshua Bengio, and Amos Storkey.
\newblock Three factors influencing minima in sgd.
\newblock \emph{arXiv preprint arXiv:1711.04623}, 2017.

\bibitem[Jiang et~al.(2019)Jiang, Neyshabur, Mobahi, Krishnan, and
  Bengio]{jiang2019fantastic}
Yiding Jiang, Behnam Neyshabur, Hossein Mobahi, Dilip Krishnan, and Samy
  Bengio.
\newblock Fantastic generalization measures and where to find them.
\newblock In \emph{International Conference on Learning Representations}, 2019.

\bibitem[Johnson \& Zhang(2013)Johnson and Zhang]{johnson2013accelerating}
Rie Johnson and Tong Zhang.
\newblock Accelerating stochastic gradient descent using predictive variance
  reduction.
\newblock \emph{Advances in neural information processing systems},
  26:\penalty0 315--323, 2013.

\bibitem[Keskar et~al.(2017)Keskar, Mudigere, Nocedal, Smelyanskiy, and
  Tang]{keskar2017large}
Nitish~Shirish Keskar, Dheevatsa Mudigere, Jorge Nocedal, Mikhail Smelyanskiy,
  and Ping Tak~Peter Tang.
\newblock On large-batch training for deep learning: Generalization gap and
  sharp minima.
\newblock In \emph{5th International Conference on Learning Representations}.
  OpenReview.net, 2017.

\bibitem[Krizhevsky(2009)]{krizhevsky2009learning}
Alex Krizhevsky.
\newblock Learning multiple layers of features from tiny images.
\newblock Technical report, University of Toronto, 2009.

\bibitem[Krizhevsky et~al.(2012)Krizhevsky, Sutskever, and
  Hinton]{krizhevsky2012imagenet}
Alex Krizhevsky, Ilya Sutskever, and Geoffrey~E Hinton.
\newblock Imagenet classification with deep convolutional neural networks.
\newblock \emph{Advances in neural information processing systems},
  25:\penalty0 1097--1105, 2012.

\bibitem[Lei \& Ying(2020)Lei and Ying]{lei2020fine}
Yunwen Lei and Yiming Ying.
\newblock Fine-grained analysis of stability and generalization for stochastic
  gradient descent.
\newblock In \emph{International Conference on Machine Learning}, pp.\
  5809--5819. PMLR, 2020.

\bibitem[London(2017)]{london2017pac}
Ben London.
\newblock A pac-bayesian analysis of randomized learning with application to
  stochastic gradient descent.
\newblock In \emph{Proceedings of the 31st International Conference on Neural
  Information Processing Systems}, pp.\  2935--2944, 2017.

\bibitem[Merity et~al.(2018)Merity, Keskar, and Socher]{merity2018regularizing}
Stephen Merity, Nitish~Shirish Keskar, and Richard Socher.
\newblock Regularizing and optimizing lstm language models.
\newblock In \emph{International Conference on Learning Representations}, 2018.

\bibitem[Nakkiran et~al.(2019)Nakkiran, Kaplun, Bansal, Yang, Barak, and
  Sutskever]{nakkiran2019deep}
Preetum Nakkiran, Gal Kaplun, Yamini Bansal, Tristan Yang, Boaz Barak, and Ilya
  Sutskever.
\newblock Deep double descent: Where bigger models and more data hurt.
\newblock In \emph{International Conference on Learning Representations}, 2019.

\bibitem[Negrea et~al.(2019)Negrea, Haghifam, Dziugaite, Khisti, and
  Roy]{negrea2019information}
Jeffrey Negrea, Mahdi Haghifam, Gintare~Karolina Dziugaite, Ashish Khisti, and
  Daniel~M Roy.
\newblock Information-theoretic generalization bounds for sgld via
  data-dependent estimates.
\newblock In \emph{Advances in Neural Information Processing Systems}, pp.\
  11013--11023, 2019.

\bibitem[Netzer et~al.(2011)Netzer, Wang, Coates, Bissacco, Wu, and
  Ng]{netzer2011reading}
Yuval Netzer, Tao Wang, Adam Coates, Alessandro Bissacco, Bo~Wu, and Andrew~Y
  Ng.
\newblock Reading digits in natural images with unsupervised feature learning.
\newblock In \emph{NeurIPS Workshop on Deep Learning and Unsupervised Feature
  Learning}, 2011.

\bibitem[Neu et~al.(2021)Neu, Dziugaite, Haghifam, and
  Roy]{Neu2021InformationTheoreticGB}
Gergely Neu, Gintare~Karolina Dziugaite, Mahdi Haghifam, and Daniel~M Roy.
\newblock Information-theoretic generalization bounds for stochastic gradient
  descent.
\newblock In \emph{COLT}, 2021.

\bibitem[Neyshabur et~al.(2015)Neyshabur, Tomioka, and
  Srebro]{neyshabur2015norm}
Behnam Neyshabur, Ryota Tomioka, and Nathan Srebro.
\newblock Norm-based capacity control in neural networks.
\newblock In \emph{Conference on Learning Theory}, pp.\  1376--1401. PMLR,
  2015.

\bibitem[Neyshabur et~al.(2017)Neyshabur, Bhojanapalli, McAllester, and
  Srebro]{NeyshaburBMS17}
Behnam Neyshabur, Srinadh Bhojanapalli, David McAllester, and Nati Srebro.
\newblock Exploring generalization in deep learning.
\newblock In \emph{Advances in Neural Information Processing Systems}, pp.\
  5947--5956, 2017.

\bibitem[Neyshabur et~al.(2018{\natexlab{a}})Neyshabur, Bhojanapalli, and
  Srebro]{neyshabur2018pac}
Behnam Neyshabur, Srinadh Bhojanapalli, and Nathan Srebro.
\newblock A pac-bayesian approach to spectrally-normalized margin bounds for
  neural networks.
\newblock In \emph{International Conference on Learning Representations},
  2018{\natexlab{a}}.

\bibitem[Neyshabur et~al.(2018{\natexlab{b}})Neyshabur, Li, Bhojanapalli,
  LeCun, and Srebro]{neyshabur2018role}
Behnam Neyshabur, Zhiyuan Li, Srinadh Bhojanapalli, Yann LeCun, and Nathan
  Srebro.
\newblock The role of over-parametrization in generalization of neural
  networks.
\newblock In \emph{International Conference on Learning Representations},
  2018{\natexlab{b}}.

\bibitem[Paszke et~al.(2019)Paszke, Gross, Massa, Lerer, Bradbury, Chanan,
  Killeen, Lin, Gimelshein, Antiga, et~al.]{paszke2019pytorch}
Adam Paszke, Sam Gross, Francisco Massa, Adam Lerer, James Bradbury, Gregory
  Chanan, Trevor Killeen, Zeming Lin, Natalia Gimelshein, Luca Antiga, et~al.
\newblock Pytorch: An imperative style, high-performance deep learning library.
\newblock \emph{Advances in Neural Information Processing Systems},
  32:\penalty0 8026--8037, 2019.

\bibitem[Pensia et~al.(2018)Pensia, Jog, and Loh]{pensia2018generalization}
Ankit Pensia, Varun Jog, and Po-Ling Loh.
\newblock Generalization error bounds for noisy, iterative algorithms.
\newblock In \emph{2018 IEEE International Symposium on Information Theory
  (ISIT)}, pp.\  546--550. IEEE, 2018.

\bibitem[Peters et~al.(2018)Peters, Neumann, Iyyer, Gardner, Clark, Lee, and
  Zettlemoyer]{peters2018deep}
Matthew Peters, Mark Neumann, Mohit Iyyer, Matt Gardner, Christopher Clark,
  Kenton Lee, and Luke Zettlemoyer.
\newblock Deep contextualized word representations.
\newblock In \emph{Proceedings of the 2018 Conference of the North American
  Chapter of the Association for Computational Linguistics: Human Language
  Technologies, Volume 1 (Long Papers)}, pp.\  2227--2237, 2018.

\bibitem[Polyanskiy \& Wu(2019)Polyanskiy and Wu]{polyanskiy2019lecture}
Yury Polyanskiy and Yihong Wu.
\newblock Lecture notes on information theory.
\newblock \emph{Lecture Notes for 6.441 (MIT), ECE 563 (UIUC), STAT 364 (Yale),
  2019.}, 2019.

\bibitem[Raginsky \& Sason(2018)Raginsky and Sason]{raginsky2018concentration}
Maxim Raginsky and Igal Sason.
\newblock Concentration of measure inequalities in information theory,
  communications and coding.
\newblock \emph{Foundations and Trends in Communications and Information
  Theory; NOW Publishers: Boston, MA, USA}, 2018.

\bibitem[Rodr{\'\i}guez-G{\'a}lvez et~al.(2020)Rodr{\'\i}guez-G{\'a}lvez,
  Bassi, Thobaben, and Skoglund]{rodriguez2020random}
Borja Rodr{\'\i}guez-G{\'a}lvez, Germ{\'a}n Bassi, Ragnar Thobaben, and Mikael
  Skoglund.
\newblock On random subset generalization error bounds and the stochastic
  gradient langevin dynamics algorithm.
\newblock \emph{arXiv preprint arXiv:2010.10994}, 2020.

\bibitem[Roux et~al.(2012)Roux, Schmidt, and Bach]{roux2012stochastic}
Nicolas~Le Roux, Mark Schmidt, and Francis Bach.
\newblock A stochastic gradient method with an exponential convergence rate for
  finite training sets.
\newblock In \emph{Proceedings of the 25th International Conference on Neural
  Information Processing Systems-Volume 2}, pp.\  2663--2671, 2012.

\bibitem[Russo \& Zou(2016)Russo and Zou]{russo2016controlling}
Daniel Russo and James Zou.
\newblock Controlling bias in adaptive data analysis using information theory.
\newblock In \emph{Artificial Intelligence and Statistics}, pp.\  1232--1240,
  2016.

\bibitem[Russo \& Zou(2019)Russo and Zou]{russo2019much}
Daniel Russo and James Zou.
\newblock How much does your data exploration overfit? controlling bias via
  information usage.
\newblock \emph{IEEE Transactions on Information Theory}, 66\penalty0
  (1):\penalty0 302--323, 2019.

\bibitem[Shalev-Shwartz \& Ben-David(2014)Shalev-Shwartz and
  Ben-David]{shalev2014understanding}
Shai Shalev-Shwartz and Shai Ben-David.
\newblock \emph{Understanding machine learning: From theory to algorithms}.
\newblock Cambridge university press, 2014.

\bibitem[Simonyan \& Zisserman(2015)Simonyan and Zisserman]{simonyan2015very}
Karen Simonyan and Andrew Zisserman.
\newblock Very deep convolutional networks for large-scale image recognition.
\newblock In \emph{3rd International Conference on Learning Representations,
  {ICLR}}, 2015.

\bibitem[Srivastava et~al.(2014)Srivastava, Hinton, Krizhevsky, Sutskever, and
  Salakhutdinov]{srivastava2014dropout}
Nitish Srivastava, Geoffrey Hinton, Alex Krizhevsky, Ilya Sutskever, and Ruslan
  Salakhutdinov.
\newblock Dropout: A simple way to prevent neural networks from overfitting.
\newblock \emph{The Journal of Machine Learning Research}, 15\penalty0
  (1):\penalty0 1929--1958, 2014.

\bibitem[Steinke \& Zakynthinou(2020)Steinke and Zakynthinou]{SteinkeZ20}
Thomas Steinke and Lydia Zakynthinou.
\newblock Reasoning about generalization via conditional mutual information.
\newblock In \emph{Conference on Learning Theory}, volume 125 of
  \emph{Proceedings of Machine Learning Research}, pp.\  3437--3452. {PMLR},
  2020.

\bibitem[Szegedy et~al.(2016)Szegedy, Vanhoucke, Ioffe, Shlens, and
  Wojna]{szegedy2016rethinking}
Christian Szegedy, Vincent Vanhoucke, Sergey Ioffe, Jon Shlens, and Zbigniew
  Wojna.
\newblock Rethinking the inception architecture for computer vision.
\newblock In \emph{Proceedings of the {IEEE} Conference on Computer Vision and
  Pattern Recognition, {CVPR}}, pp.\  2818--2826, 2016.

\bibitem[Vapnik(1998)]{Vapnik1998}
Vladimir Vapnik.
\newblock \emph{Statistical learning theory}.
\newblock Wiley, 1998.
\newblock ISBN 978-0-471-03003-4.

\bibitem[Wang et~al.(2021{\natexlab{a}})Wang, Huang, Gao, and
  Calmon]{wang2021analyzing}
Hao Wang, Yizhe Huang, Rui Gao, and Flavio Calmon.
\newblock Analyzing the generalization capability of sgld using properties of
  gaussian channels.
\newblock \emph{Advances in Neural Information Processing Systems}, 34,
  2021{\natexlab{a}}.

\bibitem[Wang et~al.(2021{\natexlab{b}})Wang, Huang, Gao, and
  Calmon]{wang2021learning}
Hao Wang, Yizhe Huang, Rui Gao, and Flavio~P Calmon.
\newblock Learning while dissipating information: Understanding the
  generalization capability of sgld.
\newblock \emph{arXiv preprint arXiv:2102.02976}, 2021{\natexlab{b}}.

\bibitem[Wei et~al.(2020)Wei, Kakade, and Ma]{wei2020implicit}
Colin Wei, Sham Kakade, and Tengyu Ma.
\newblock The implicit and explicit regularization effects of dropout.
\newblock In \emph{International Conference on Machine Learning}, pp.\
  10181--10192. PMLR, 2020.

\bibitem[Welling \& Teh(2011)Welling and Teh]{welling2011bayesian}
Max Welling and Yee~W Teh.
\newblock Bayesian learning via stochastic gradient langevin dynamics.
\newblock In \emph{Proceedings of the 28th international conference on machine
  learning (ICML-11)}, pp.\  681--688. Citeseer, 2011.

\bibitem[Wen et~al.(2020)Wen, Luk, Gazeau, Zhang, Chan, and
  Ba]{wen2020empirical}
Yeming Wen, Kevin Luk, Maxime Gazeau, Guodong Zhang, Harris Chan, and Jimmy Ba.
\newblock An empirical study of stochastic gradient descent with structured
  covariance noise.
\newblock In \emph{International Conference on Artificial Intelligence and
  Statistics}, pp.\  3621--3631. PMLR, 2020.

\bibitem[Wu et~al.(2018)Wu, Ma, and Weinan]{wu2018sgd}
Lei Wu, Chao Ma, and E~Weinan.
\newblock How sgd selects the global minima in over-parameterized learning: A
  dynamical stability perspective.
\newblock \emph{Advances in Neural Information Processing Systems},
  2018:\penalty0 8279--8288, 2018.

\bibitem[Xie et~al.(2021{\natexlab{a}})Xie, He, Fu, Sato, Tao, and
  Sugiyama]{xie2021artificial}
Zeke Xie, Fengxiang He, Shaopeng Fu, Issei Sato, Dacheng Tao, and Masashi
  Sugiyama.
\newblock Artificial neural variability for deep learning: on overfitting,
  noise memorization, and catastrophic forgetting.
\newblock \emph{Neural computation}, 33\penalty0 (8):\penalty0 2163--2192,
  2021{\natexlab{a}}.

\bibitem[Xie et~al.(2021{\natexlab{b}})Xie, Yuan, Zhu, and
  Sugiyama]{xie2021positive}
Zeke Xie, Li~Yuan, Zhanxing Zhu, and Masashi Sugiyama.
\newblock Positive-negative momentum: Manipulating stochastic gradient noise to
  improve generalization.
\newblock In \emph{International Conference on Machine Learning}, pp.\
  11448--11458. PMLR, 2021{\natexlab{b}}.

\bibitem[Xu \& Raginsky(2017)Xu and Raginsky]{xu2017information}
Aolin Xu and Maxim Raginsky.
\newblock Information-theoretic analysis of generalization capability of
  learning algorithms.
\newblock \emph{Advances in Neural Information Processing Systems},
  2017:\penalty0 2525--2534, 2017.

\bibitem[Yang et~al.(2020)Yang, Yu, You, Steinhardt, and
  Ma]{yang2020rethinking}
Zitong Yang, Yaodong Yu, Chong You, Jacob Steinhardt, and Yi~Ma.
\newblock Rethinking bias-variance trade-off for generalization of neural
  networks.
\newblock In \emph{International Conference on Machine Learning}, pp.\
  10767--10777. PMLR, 2020.

\bibitem[Yao et~al.(2020)Yao, Gholami, Keutzer, and Mahoney]{yao2020pyhessian}
Zhewei Yao, Amir Gholami, Kurt Keutzer, and Michael~W Mahoney.
\newblock Pyhessian: Neural networks through the lens of the hessian.
\newblock In \emph{2020 IEEE International Conference on Big Data (Big Data)},
  pp.\  581--590. IEEE, 2020.

\bibitem[Zhang et~al.(2017)Zhang, Bengio, Hardt, Recht, and
  Vinyals]{ZhangBHRV17}
Chiyuan Zhang, Samy Bengio, Moritz Hardt, Benjamin Recht, and Oriol Vinyals.
\newblock Understanding deep learning requires rethinking generalization.
\newblock In \emph{5th International Conference on Learning Representations}.
  OpenReview.net, 2017.

\bibitem[Zhang et~al.(2018)Zhang, Cisse, Dauphin, and
  Lopez{-}Paz]{zhang2018mixup}
Hongyi Zhang, Moustapha Cisse, Yann~N. Dauphin, and David Lopez{-}Paz.
\newblock mixup: Beyond empirical risk minimization.
\newblock In \emph{6th International Conference on Learning Representations,
  {ICLR}}, 2018.

\bibitem[Zhang et~al.(2019)Zhang, He, Sra, and Jadbabaie]{zhang2019gradient}
Jingzhao Zhang, Tianxing He, Suvrit Sra, and Ali Jadbabaie.
\newblock Why gradient clipping accelerates training: A theoretical
  justification for adaptivity.
\newblock In \emph{International Conference on Learning Representations}, 2019.

\bibitem[Zheng et~al.(2021)Zheng, Zhang, and Mao]{zheng2020regularizing}
Yaowei Zheng, Richong Zhang, and Yongyi Mao.
\newblock Regularizing neural networks via adversarial model perturbation.
\newblock In \emph{{CVPR}}, 2021.

\bibitem[Zhou et~al.(2020)Zhou, Tian, and Liu]{zhou2020individually}
Ruida Zhou, Chao Tian, and Tie Liu.
\newblock Individually conditional individual mutual information bound on
  generalization error.
\newblock \emph{arXiv preprint arXiv:2012.09922}, 2020.

\end{thebibliography}
\newpage
\appendix
\addcontentsline{toc}{section}{Appendix} 
\part{Appendix} 
\parttoc
\setcounter{equation}{0}

\section{HWI Inequality: Proof of Lemma \ref{lem:HWI}}
\label{prf:HWI}
The proof given here is a simple extension of the proof of \citet[Lemma~3.4.2]{raginsky2018concentration}, which is a special instance of the (weak) HWI inequality.
\begin{proof}

\begin{align}
    \mathrm{D_{KL}}(P_{X+\sqrt{t}N}||P_{Y+\sqrt{t'}N})\leq  &\mathrm{D_{KL}}(P_{X,Y,X+\sqrt{t}N}||P_{X,Y,Y+\sqrt{t'}N}) \label{ineq:kl-chainrule}\\
    = &  \ex{X,Y}{\mathrm{D_{KL}}(P_{X+\sqrt{t}N|X,Y}||P_{Y+\sqrt{t'}N|X,Y})}\notag\\
    = &  \ex{X,Y}{\mathrm{D_{KL}}(\mathcal{N}(X,t\mathrm{I}_d)||\mathcal{N}(Y,t'\mathrm{I}_d))}\label{eq:independence}\\
    =&\frac{1}{2t'}\ex{X,Y}{||X-Y||^2}+\frac{d}{2}(\log\frac{t'}{t}+\frac{t}{t'}-1),
    \label{eq:multi-gaussian}
\end{align}
where Eq.\ref{ineq:kl-chainrule} is by the chain rule of the KL divergence, Eq.\ref{eq:independence} holds since $N$ is independent of $(X,Y)$ and Eq.\ref{eq:multi-gaussian} is by the equality
\[
\mathrm{D_{KL}}(\mathcal{N}(\mu_1,\Sigma_1)||\mathcal{N}(\mu_2,\Sigma_2))=\frac{1}{2}\left[\log{\frac{|\Sigma_2|}{|\Sigma_1|}}-d+\mathrm{Tr}\left(\Sigma_2^{-1}\Sigma_1\right)+(\mu_2-\mu_1)^T\Sigma_2^{-1}(\mu_2-\mu_1)\right].
\]
Eq.\ref{eq:multi-gaussian} holds for any joint distribution of $(X,Y)$. If we let $t'=t$, by the definition of Wasserstein distance, we have
\begin{align*}
    \mathrm{D_{KL}}(P_{X+\sqrt{t}N}||P_{Y+\sqrt{t}N})\leq\frac{1}{2t}\mathbb{W}_2^2(P_X,P_Y) .
\end{align*}
This recovers the vector version of  \citet[Lemma~3.4.2]{raginsky2018concentration}.
\end{proof}

\section{A Key Lemma: Lemma \ref{lem:keylemma}}

\begin{proof}
We give two proofs of this lemma, the first proof is inspired by the techniques used in \cite{wang2021analyzing} and \cite{guo2005mutual}, and the second proof is similar to \cite{pensia2018generalization}. 
\begin{enumerate}
    \item 
Let $Z=Y+\Delta$. By the properties of mutual information,
\[
I(f(Z,X)+\sigma N;X|Y=y)=I({f(Z,X)-\Omega}+ {\sigma}N;X|Y=y).
\]

Let $g(Z,X) \triangleq {f(Z,X)-\Omega}$. Let $P_{N'}=\mathcal{N}(0,\sigma'^2\mathbf{I}_d)$ for any $\sigma'>0$ and $N'$ is independent of other random variables for a fixed $y$. Then, for $Y=y$,
\begin{align}
    I(g(Z,X)+\sigma N;X|Y=y)=&\ex{X}{\mathrm{D_{KL}}\left(P_{g(Z,x)+{\sigma}N|X=x,Y=y}||P_{g(Z,X)+{\sigma}N|Y=y}\right)}\notag\\
    =& \ex{X}{\mathrm{D_{KL}}\left(P_{g(Z,x)+{\sigma}N|X=x,Y=y}||P_{N'|Y=y}\right)-\mathrm{D_{KL}}\left(P_{g(Z,X)+{\sigma}N|Y=y}||P_{N'|Y=y}\right)}\label{eq:golden}\\
    \leq& \inf_{\sigma'>0}\ex{X}{\mathrm{D_{KL}}\left(P_{g(Z,x)+{\sigma}N|X=x,Y=y}||P_{{\sigma'}N|Y=y}\right)}\label{ineq:klnonneg}\\
    =& \inf_{\sigma'>0}\frac{1}{2\sigma'^2}\ex{X}{\ex{\Delta}{||g(y+\Delta,x)||^2|X=x,Y=y}}\label{ineq:hwi-origin}+\frac{d}{2}(\log{\frac{\sigma'^2}{\sigma^2}}+\frac{\sigma^2}{\sigma'^2}-1)\\
    =&\frac{d}{2}\log{\left(\frac{\ex{X,\Delta}{||g(y+\Delta,X)||^2|Y=y}}{d\sigma^2}+1\right)},\notag
\end{align}
where Eq. \ref{eq:golden} is the golden formula (see Theorem 3.3 in \cite{polyanskiy2019lecture} ), Eq. \ref{ineq:klnonneg} is by the non-negativity of relative entropy, Eq. \ref{ineq:hwi-origin} is due to Lemma \ref{lem:HWI} and the last equality holds when the optimal \[\sigma'^2 = \frac{\ex{X,\Delta}{||g(y+\Delta,X)||^2|Y=y}}{d}+\sigma^2\]
is achieved.

Thus,
\begin{align*}
    I(f(Z,X)+N;X|Y)&= \ex{Y}{I(f(Z,X)+N;X|Y=y)}\\
    &\leq\ex{Y}{\frac{d}{2}\log{\left(\frac{\ex{X,\Delta}{||g(y+\Delta,X)||^2|Y=y}}{d\sigma^2}+1\right)}}\\
    &=\ex{Y}{\frac{d}{2}\log{\left(\frac{\ex{X,\Delta}{||f(Z,X)-\Omega||^2|Y=y}}{d\sigma^2}+1\right)}},
\end{align*}
which completes the proof.

\item Alternatively, let $h(\cdot)$ denote the differential entropy, then
\begin{align}
    I(f(Z,X)+\sigma N;X|Y=y)&=I(g(Z,X)+\sigma N;X|Y=y)\notag\\
    &=h(g(Z,X)+\sigma N|Y=y)-h(g(Z,X)+\sigma N|X,Y=y)\notag\\
    &\leq h(g(Z,X)+\sigma N|Y=y)-h(g(Z,X)+\sigma N|\Delta,X,Y=y)\label{ineq:conditional-entropy}\\
    &= h(g(Z,X)+\sigma N|Y=y)-h(\sigma N)\notag\\
    &= h(g(Z,X)+\sigma N|Y=y)-\frac{d}{2}\log{2\pi e \sigma^2}\notag,
\end{align}
where Eq. \ref{ineq:conditional-entropy} is due to the fact that conditioning reduces entropy.

Notice that 
\[
\ex{\Delta,X,N}{||g(y+\Delta,X)+\sigma N||^2|Y=y}=\ex{\Delta,X}{||g(y+\Delta,X)||^2|Y=y}+d\sigma^2.
\]

Since the Gaussian distribution maximizes the entropy over all distributions with the same variance \citep{cover2012elements}, we have
\[
h(g(Z,X)+\sigma N|Y=y)\leq \frac{d}{2}\log{2\pi e \frac{\ex{\Delta,X}{||g(y+\Delta,X)||^2|Y=y}+d\sigma^2}{d}}.
\]

Therefore,
\begin{align}
    I(f(Z,X)+\sigma N;X|Y=y)&\leq \frac{d}{2}\log{2\pi e \frac{\ex{\Delta,X}{||g(y+\Delta,X)||^2|Y=y}+d\sigma^2}{d}}-\frac{d}{2}\log{2\pi e \sigma^2}\notag\\
    &=\frac{d}{2}\log{\left( \frac{\ex{\Delta,X}{||g(y+\Delta,X)||^2|Y=y}}{d\sigma^2}+1\right)}.\notag
\end{align}

The remaining part is straightforward and is the same with the previous proof.

\end{enumerate}
\end{proof}

\section {Mutual Information Bounds for SGD: Proof of Theorem \ref{thm:tightest-bound} and Theorem \ref{thm:re-neu-bound}}

In this section, we provide a proof for Theorem \ref{thm:tightest-bound} and Theorem \ref{thm:re-neu-bound}, while elaborating on why the local gradient sensitivity term in Lemma \ref{lem:neu-bound} can be removed. The key ingredient that makes this possible is a construction given in the proof of 
Lemma \ref{lem:sample-cmi-var}.

\subsection{Lemma \ref{lem:sgd-chain}}
We first give the following useful lemma.
\begin{lem}
\label{lem:sgd-chain}
The mutual information $I(\widetilde{W}_T;S)\leq\sum_{t=1}^T I\left(-\lambda_tg(W_{t-1},B_t)+N_t;S|\widetilde{W}_{t-1}\right)$.
\end{lem}
\begin{proof}
The mutual information between the final output of SGD and the training sample can be upper bounded by the mutual information between the full trajectories and the training sample, which is shown below.
\begin{align}
    &I(\widetilde{W}_T;S)\notag\\
    =&I\left(\widetilde{W}_{T-1}-\lambda_Tg(W_{T-1},B_T)+N_T;S\right)\notag\\
    \leq& I\left(\widetilde{W}_{T-1},-\lambda_Tg(W_{T-1},B_T)+N_T;S\right)\label{ineq:sample-func-dpi}\\
    =&I(\widetilde{W}_{T-1};S)+I\left(-\lambda_Tg(W_{T-1},B_T)+N_T;S|\widetilde{W}_{T-1}\right)\label{eq:sample-chain-rule}\\
    \leq& I(\widetilde{W}_{T-2};S)+I\left(-\lambda_{T-1}g(W_{T-2},B_{T-1})+N_{T-1};S|\widetilde{W}_{T-2}\right)+I\left(-\lambda_Tg(W_{T-1},B_T)+N_T;S|\widetilde{W}_{T-1}\right)\label{ineq:sample-repeat}\\
    \vdots&\notag\\
    \leq&I(\widetilde{W}_0;S)+\sum_{t=1}^T I\left(-\lambda_tg(W_{t-1},B_t)+N_t;S|\widetilde{W}_{t-1}\right),
    \label{ineq:sample-lasttwo}\\
    =&\sum_{t=1}^T I\left(-\lambda_tg(W_{t-1},B_t)+N_t;S|\widetilde{W}_{t-1}\right),
    \label{ineq:sample-last}
\end{align}
where Eq.\ref{ineq:sample-func-dpi} is by $I(f(X,Y);Z)\leq I(X,Y;Z)$ (i.e. $Z-(X,Y)-f(X,Y)$ forms a Markov chain and then use the data processing inequality), Eq.\ref{eq:sample-chain-rule} is by the chain rule of mutual information, Eq.\ref{ineq:sample-repeat} is by applying the similar procedure (namely, Eq.\ref{ineq:sample-func-dpi}-Eq.\ref{eq:sample-chain-rule}) to $I(\widetilde{W}_{T-1};S)$, Eq. \ref{ineq:sample-lasttwo} is by doing these steps recursively and Eq.\ref{ineq:sample-last} is due to the fact that $\widetilde{W}_0$ is independent of $S$ (i.e. $I(\widetilde{W}_0;S)=0$).
\end{proof}

\subsection{Proof of Theorem \ref{thm:tightest-bound}}
\begin{proof}
We first follow the similar decomposition of the expected generalization error and apply Lemma \ref{lem:xu's-bound} as in \cite{Neu2021InformationTheoreticGB},
\begin{align}
    |{\rm gen}(\mu,P_{W_T|S})|&=\left|{\rm gen}(\mu, P_{\widetilde{W}_T|S})+\ex{W_T,\Delta_T}{L_\mu(W_T)-L_\mu(\widetilde{W}_T)}+\ex{W_T,\Delta_T,S}{L_S(\widetilde{W}_T)-L_S(W_T)}\right|\notag\\
    &\leq\sqrt{\frac{2R^2}{n}I(\widetilde{W}_T;S)}+\left|\ex{W_T,S,S'}{\gamma(W_T,S)-\gamma(W_T,S')}\right|\label{ineq:xu-bound}.
\end{align}

The remaining task is to bound the mutual information $I(\widetilde{W}_T;S)$. 

Let $X=S$, $Y=\widetilde{W}_{t-1}$ and $\Omega = \ex{\Delta_{t-1},Z}{g(\widetilde{w}_{t-1}-\Delta_{t-1},Z)}$, by applying Lemma \ref{lem:keylemma}, we have
\begin{align*}
&I\left(-\lambda_tg(W_{t-1},B_t)+N_t;S|\widetilde{W}_{t-1}\right)\\\leq&\ex{\widetilde{W}_{t-1}}{\frac{d}{2}\log{\left(\frac{\lambda_t^2\ex{S,\Delta_{t-1}}{||g(\widetilde{w}_{t-1}-\Delta_{t-1},B_t)-\ex{\Delta_{t-1},Z}{g(\widetilde{w}_{t-1}-\Delta_{t-1},Z)}||^2|\widetilde{W}_{t-1}=\widetilde{w}_{t-1}}}{d\sigma^2}+1\right)}}.    
\end{align*}

By Lemma \ref{lem:sgd-chain} and putting everything together, we have
\begin{align*}
&|\mathrm{gen}(\mu,P_{W_T|S})|\\\leq& \sqrt{\frac{R^2d}{n}\sum_{t=1}^T\ex{\widetilde{W}_{t-1}}{\log\left(\frac{\lambda_t^2\ex{S,\Delta_{t-1}}{||g(W_{t-1},B_t)-\ex{Z,\Delta_{t-1}}{g(W_{t-1},Z)}||^2}}{d\sigma_t^2}+1\right)}}+\left|\ex{}{\gamma(W_T,S)-\gamma(W_T,S')}\right|.
\end{align*}
This completes the proof.

\end{proof}

\subsection{Corollary \ref{cor:recover-bound}: Recovering Lemma \ref{lem:neu-bound}}
\label{appendix:recover-bound}
We can recover the result of Lemma \ref{lem:neu-bound} in \cite{Neu2021InformationTheoreticGB} from Theorem \ref{thm:tightest-bound}, which is re-stated in the following corollary.

\begin{cor}
\label{cor:recover-bound}
The generalization error of SGD is upper bounded by
\begin{align*} 
|\mathrm{gen}(\mu,P_{W_T|S})|\leq\sqrt{\frac{2R^2}{n}\sum_{t=1}^T{\frac{\lambda_t^2}{\sigma_t^2}}\ex{W_{t-1}}{3\Psi(W_{t-1})+ 2\mathbb{\widetilde{V}}_t(W_{t-1})}}+\left|\ex{}{\gamma(W_T,S)-\gamma(W_T,S')}\right|.
\end{align*}
\end{cor}

\begin{proof}
From the first term in Theorem \ref{thm:tightest-bound}, we have
\begin{align}
&\sqrt{\frac{R^2d}{n}\sum_{t=1}^T\ex{\widetilde{W}_{t-1}}{\log\left(\frac{\lambda_t^2\ex{S,\Delta_{t-1}}{\left|\left|g(W_{t-1},B_t)-\ex{Z,\Delta_{t-1}}{g(W_{t-1},Z)}\right|\right|^2}}{d\sigma_t^2}+1\right)}}\notag\\
\leq&\sqrt{\frac{R^2d}{n}\sum_{t=1}^T{\log\left(\frac{\lambda_t^2\ex{\widetilde{W}_{t-1},\Delta_{t-1},S}{\left|\left|g(W_{t-1},B_t)-\ex{Z,\Delta_{t-1}}{g(\widetilde{W}_{t-1}-\Delta_{t-1},Z)}\right|\right|^2}}{d\sigma_t^2}+1\right)}}\label{ineq:jsineq-recover}\\
\leq&\sqrt{\frac{R^2}{n}\sum_{t=1}^T{\frac{\lambda_t^2}{\sigma_t^2}}\ex{\widetilde{W}_{t-1},\Delta_{t-1},S}{\left|\left|g(W_{t-1},B_t)-\ex{Z,\Delta_{t-1}}{g(\widetilde{W}_{t-1}-\Delta_{t-1},Z)}\right|\right|^2}}\label{ineq:logx-recover}\\
=&\sqrt{\frac{R^2}{n}\sum_{t=1}^T{\frac{\lambda_t^2}{\sigma_t^2}}\ex{W_{t-1},S}{\left|\left|g(W_{t-1},B_t)-\ex{Z}{g(\widetilde{W}_{t-1},Z)}+\ex{Z}{g(\widetilde{W}_{t-1},Z)}-\ex{Z,\Delta_{t-1}}{g(\widetilde{W}_{t-1}-\Delta_{t-1},Z)}\right|\right|^2}}\notag\\
\leq&\sqrt{\frac{R^2}{n}\sum_{t=1}^T{\frac{\lambda_t^2}{\sigma_t^2}}\ex{{W}_{t-1},S}{2\left|\left|g(W_{t-1},B_t)-\ex{Z}{g(\widetilde{W}_{t-1},Z)}\right|\right|^2+2\left|\left|\ex{Z}{g(\widetilde{W}_{t-1},Z)}-\ex{Z,\Delta_{t-1}}{g(\widetilde{W}_{t-1}-\Delta_{t-1},Z)}\right|\right|^2}}\label{ineq:triangle-recover}\\
\leq&\sqrt{\frac{R^2}{n}\sum_{t=1}^T{\frac{\lambda_t^2}{\sigma_t^2}}\ex{W_{t-1},S}{2\left|\left|g({W}_{t-1},B_t)-\ex{Z}{g(\widetilde{W}_{t-1},Z)}\right|\right|^2+2\ex{\Delta_{t-1}}{\left|\left|\ex{Z}{g(\widetilde{W}_{t-1},Z)}-\ex{Z}{g(\widetilde{W}_{t-1}-\Delta_{t-1},Z)}\right|\right|^2}}}\label{ineq:squarednorm-convex}\\
=&\sqrt{\frac{R^2}{n}\sum_{t=1}^T{\frac{\lambda_t^2}{\sigma_t^2}}\ex{W_{t-1},S}{2\left|\left|g({W}_{t-1},B_t)-\ex{Z}{g({W}_{t-1},Z)}+\ex{Z}{g({W}_{t-1},Z)}-\ex{Z}{g(\widetilde{W}_{t-1},Z)}\right|\right|^2+2\Psi(W_{t-1})}}\notag\\
\leq&\sqrt{\frac{R^2}{n}\sum_{t=1}^T{\frac{\lambda_t^2}{\sigma_t^2}}\ex{W_{t-1},S}{4\left|\left|g({W}_{t-1},B_t)-\ex{Z}{g({W}_{t-1},Z)}\right|\right|^2+4\left|\left|\ex{Z}{g({W}_{t-1},Z)}-\ex{Z}{g(\widetilde{W}_{t-1},Z)}\right|\right|^2+2\Psi(W_{t-1})}}\label{ineq:triangle-recover-2}\\
=&\sqrt{\frac{R^2}{n}\sum_{t=1}^T{\frac{\lambda_t^2}{\sigma_t^2}}\ex{W_{t-1},S}{4\left|\left|g({W}_{t-1},B_t)-\ex{Z}{g({W}_{t-1},Z)}\right|\right|^2+6\Psi(W_{t-1})}}\notag\\
=&\sqrt{\frac{2R^2}{n}\sum_{t=1}^T{\frac{\lambda_t^2}{\sigma_t^2}}\ex{W_{t-1}}{3\Psi(W_{t-1})+ 2\mathbb{\widetilde{V}}_t(W_{t-1})}},\notag
\end{align}
where Eq. \ref{ineq:jsineq-recover} and \ref{ineq:squarednorm-convex} are by Jensen's inequality (i.e. the concavity of logarithm and the convexity of squared norm), Eq. \ref{ineq:logx-recover} is by $\log(x+1)\leq x$, and Eq. \ref{ineq:triangle-recover} and \ref{ineq:triangle-recover-2} are by $||x+y||^2\leq2||x||^2+2||y||^2$.

This completes the proof.
\end{proof}

\subsection{Proof of Lemma \ref{lem:sample-cmi-var}}
\label{sec:neu-diff}
\begin{proof}
We first notice that
\[
I(G_t+N_t;S|\widetilde{W}_{t-1})=I(-\lambda_t g(W_{t-1},B_t)+\sigma_t N; S|\widetilde{W}_{t-1}),
\]
where $N\sim \mathcal{N}(0,\mathrm{I}_d)$.

Then let $X=S$, $Y=\widetilde{W}_{t-1}$ and $\Omega=\ex{W_{t-1},Z}{\nabla\ell(W_{t-1},Z)}$, by applying Lemma \ref{lem:keylemma}, we have
\begin{align*}
    I(-\lambda_t g(W_{t-1},B_t)+\sigma_t N; S|\widetilde{W}_{t-1})\leq& \frac{d}{2}\ex{\widetilde{W}_{t-1}}{\log\left(\frac{\lambda_t^2\ex{S,\Delta_{t-1}}{||g(W_{t-1},B_t)-\ex{W_{t-1},Z}{\nabla_w\ell(W_{t-1},Z)}||^2}}{d\sigma_t^2}+1\right)}\\
    \leq&\frac{d}{2}{\log\left(\frac{\lambda_t^2\ex{W_{t-1}}{\mathbb{V}_{t}(W_{t-1})}}{d\sigma_t^2}+1\right)},
\end{align*}
where the second inequality is by Jensen's inequality.

This completes the proof.
\end{proof}

\subsection{Proof of Theorem \ref{thm:re-neu-bound}}

\begin{proof}


Applying Lemma \ref{lem:sgd-chain} and Lemma \ref{lem:sample-cmi-var} and putting everything together, we have
\begin{align*}
|\mathrm{gen}(\mu,P_{W_T|S})|\leq \sqrt{\frac{R^2d}{n}\sum_{t=1}^T{\log\left(\frac{\lambda_t^2\ex{W_{t-1}}{\mathbb{V}_{t}(W_{t-1})}}{d\sigma_t^2}+1\right)}}+\left|\ex{}{\gamma(W_T,S)-\gamma(W_T,S')}\right|.
\end{align*}

Next, to handle the mismatch between the outputs of perturbed SGD and SGD, we apply Taylor expansion around $\Delta_T=\vec 0$,
\begin{align}
    \ex{W_T,S,\Delta_T}{L_S(W_T+\Delta_T)-L_S(W_T)}&=\frac{1}{n}\sum_{i=1}^n \ex{W_T,Z_i,\Delta_T}{\ell(W_T+\Delta_T,Z_i)-\ell(W_T,Z_i)}\notag\\
    &\approx  \ex{W_T,Z,\Delta_T}{\langle \nabla_w \ell(W_T,Z), \Delta_T \rangle + \frac{1}{2}\Delta_T^T\mathrm{H}_{W_T}(Z)\Delta_T}\notag\\
    & = \ex{W_T,Z,\Delta_T}{\frac{1}{2}\Delta_T^T\mathrm{H}_{W_T}(Z)\Delta_T}\label{eq:noise-mean}\\
    &=\frac{1}{2}\langle \ex{W_T,Z}{\mathrm{H}_{W_T}(Z)}, \ex{\Delta_T}{\Delta_T\Delta_T^T}\rangle\notag\\
    &=\frac{1}{2}\langle \ex{W_T,Z}{\mathrm{H}_{W_T}(Z)}, \mathrm{diag}(\sum_{t=1}^T \sigma_t^2)\rangle\label{eq:noise-independent}\\
    &=\frac{\sum_{t=1}^T \sigma_t^2}{2}\mathrm{Tr} (\ex{W_T,Z}{\mathrm{H}_{W_T}(Z)}),\notag
\end{align}
where Eq.\ref{eq:noise-mean} is by the zero mean of the perturbation, Eq.\ref{eq:noise-independent} is by the independence of the coordinates of $\Delta_T$, $\langle \cdot,\cdot \rangle$ denotes the inner product of two matrices, $\mathrm{diag}(A)$ is the diagonal matrix with element $A$ and $\mathrm{Tr}(\cdot)$ is the trace of a matrix.

Under the condition $\ex{W_T,S'}{\gamma(W_T,S')}\geq 0$, we now bound $\mathrm{gen}(\mu,{\rm P}_{\widetilde{W}_T|S})$ instead of its absolute value,  $|\mathrm{gen}(\mu,{\rm P}_{\widetilde{W}_T|S})|$. With the inequality $\log(x+1)\leq x$, the following is straightforward,
\begin{align}
    \mathrm{gen}(\mu,{\rm P}_{\widetilde{W}_T|S})\leq& \sqrt{\frac{R^2}{n}\sum_{t=1}^T\frac{\lambda_t^2}{\sigma_t^2}\ex{W_{t-1}}{\mathbb{V}(W_{t-1})}}+\ex{W_T,S,S'}{\gamma(W_T,S)-\gamma(W_T,S')}\notag\\
    \leq & \sqrt{\frac{R^2}{n}\sum_{t=1}^T\frac{\lambda_t^2}{\sigma_t^2}\ex{W_{t-1}}{\mathbb{V}(W_{t-1})}}+\ex{W_T,S}{\gamma(W_T,S)}\notag\\
    =& \sqrt{\frac{R^2}{n}\sum_{t=1}^T\frac{\lambda_t^2}{\sigma_t^2}\ex{W_{t-1}}{\mathbb{V}(W_{t-1})}}+\frac{\sum_{t=1}^T \sigma_t^2}{2}\mathrm{Tr} (\ex{W_T,Z}{\mathrm{H}_{W_T}(Z)}).\notag
\end{align}

Since every choice of $\sigma$ gives a valid generalization bound. The optimal bound can be found by simply utilizing the fact $A/\sigma+\sigma^2B\geq 3(A/2)^{2/3}B^{1/3}$ for any positive $A$ and $B$, where the equality is achieved by the optimal $\sigma$. 
Finally, rearranging the terms will complete the proof.
\end{proof}

\subsection{Corollary \ref{cor:smooth-bound}}
\label{prf:smooth-bound}

\begin{cor}
If the loss function is differentiable and $\beta$-smooth with respect to $w$, then, 
\begin{align*}
|\mathrm{gen}(\mu,P_{W_T|S})|\leq&  
\sqrt{\frac{R^2d}{n}\sum_{t=1}^T{\log\left(\frac{\lambda_t^2\ex{W_{t-1}}{\mathbb{V}_{t}(W_{t-1})}}{d\sigma_t^2}+1\right)}}
+\beta d \sum_{t=1}^T\sigma_t^2.
\end{align*}
\label{cor:smooth-bound}
\end{cor}

\begin{proof}
Recall the smoothness implies $f(\mathbf{v})\leq f(\mathbf{w}) + \langle \nabla f(\mathbf{w}), \mathbf{v}-\mathbf{w}\rangle + \frac{\beta}{2}||\mathbf{v}-\mathbf{w}||^2$ for all $\mathbf{v}$ and $\mathbf{w}$. By the triangle inequality, we have 
\[
    |\mathbb{E}\left[L_\mu(W_T)-L_\mu(W_T+ \Delta_T)\right]|\leq |\mathbb{E}\left[ \langle \nabla_w\ell(W_T,Z), \Delta_T \rangle \right]|+\frac{\beta}{2}\mathbb{E}\left[||\Delta_T||^2\right]=\frac{\beta d \sum_{t=1}^T\sigma_t^2}{2}
\]
Thus, we can see that $|\mathbb{E}\left[L_\mu(W_T)-L_\mu(W_T+ \Delta_T)\right]|+|\mathbb{E}\left[L_S(W_T+ \Delta_T)-L_S(W_T)\right]|\leq \beta d \sum_{t=1}^T\sigma_t^2$.

This completes the proof.
\end{proof}

\begin{rem}
In Corollary \ref{cor:smooth-bound}, we note that the dependence of $d$ in the bound results from the spherical Gaussian noise used in the construction of the weight process $\widetilde{W}_T$. It is possible to replace the spherical Gaussian with a Gaussian noise having a non-diagonal covariance that reflects the geometry of the loss landscape. With this replacement, the dimension $d$ in the flatness term will be replaced by the trace of $\sum_{t=1}^T\kappa_t$, where $\kappa_t$ is the covariance matrix of the noise added at step $t$. Please refer to \cite{Neu2021InformationTheoreticGB} for a similar development.
\end{rem}

\section{Application in Neural Networks: Proof of Theorems \ref{thm:optimal-bound-linear} and \ref{thm:optimal-bound-relu}}

\subsection{Lemma \ref{lem:sgd-bound-norm}}
\begin{lem}
Under the same conditions of Theorem \ref{thm:re-neu-bound}, the generalization error of SGD is upper bounded by
\[
    \mathrm{gen}(\mu,P_{W_T|S})\leq\frac{3}{2}\left(\sum_{t=1}^T\frac{R^2\lambda_t^2T}{n} \ex{}{||g(W_{t-1}, B_t)||^2_2}\ex{}{\mathrm{Tr}\left({\mathrm{H}_{W_T}(Z)} \right)}\right)^{\frac{1}{3}}
\]
\label{lem:sgd-bound-norm}
\end{lem}
\begin{proof}
The proof of this lemma follows the same steps in the proof of Theorem \ref{thm:re-neu-bound} except that we require a different use of Lemma \ref{lem:keylemma} here. 

Specially, we let $\Omega=0$ in Lemma \ref{lem:keylemma}. The remaining steps are the same in the proof of Theorem \ref{thm:re-neu-bound} and should be straightforward.
\end{proof}

\begin{rem}
The bound in Lemma \ref{lem:sgd-bound-norm} is weaker than Eq. \ref{eq:optimal-bound} in Theorem \ref{thm:re-neu-bound} as the centralized expected vector norm should be smaller than the original expected vector norm.
\end{rem}

\subsection{Proof of Theorem \ref{thm:optimal-bound-linear}}

\begin{proof}
By Cauchy–Schwarz inequality, we have 
\begin{align}
    \ex{}{||g(W_{t-1},B_t)||_2^2}=\ex{}{\left|\left|\frac{1}{b}\sum_{Z_i\in B_t}\nabla_w\ell(W_{t-1},Z_i)\right|\right|_2^2}\leq\ex{}{||\nabla_w\ell(W_{t-1},Z)||_2^2}.\label{ineq:c-s}
\end{align}

Notice that $\nabla_w\ell(W,Z) = (W^TX-Y)X$. Let $\hat{Y}=f(W,X)$. Then,

\begin{align}
    ||\nabla_w\ell(W_{t-1},Z)||_2^2=||(W_{t-1}^TX-Y)X||_2^2\leq||W_{t-1}^TX-Y||_2^2=(\hat{Y}-Y)^2=2\ell(W_{t-1},Z).\label{ineq:ln-gradient}
\end{align}


For the Hessian matrix, it's easy to see that $\mathrm{Tr}(\mathrm{H}_{W_T}(Z))=1$.


Plugging everything into Lemma \ref{lem:sgd-bound-norm}, we have
\[
    \mathrm{gen}(\mu,P_{W_T|S})\leq3\left(\sum_{t=1}^T\frac{R^2\lambda_t^2T}{4n} \ex{W_{t-1},Z}{\ell(W_{t-1},Z)}\right)^{\frac{1}{3}}.
\]
This completes the proof.
\end{proof}

\subsection{Proof of Theorem \ref{thm:optimal-bound-relu}}
\begin{proof}

Since $ \nabla_{w_r}\ell(W,Z_i)=\frac{1}{\sqrt{m}}A_r(f(W,X_i)-Y_i)X_i\mathbb{I}_{r,i}$, where $\mathbb{I}_{r,i}=\mathbb{I}\{W_r^TX_i\geq 0\}$, we have

\begin{align}
    &||\nabla_w\ell(W_{t-1},Z)||_2^2\\
    =&\sum_{r=1}^m||\frac{1}{\sqrt{m}}A_r(f(W_{t-1},X_i)-Y_i)X_i\mathbb{I}_{r,i}||_2^2\notag\\
    \leq&\sum_{r=1}^m||\frac{1}{\sqrt{m}}A_r(f(W_{t-1},X_i)-Y_i)\mathbb{I}_{r,i}||_2^2\notag\\
    =& \frac{1}{m}\sum_{r=1}^m\mathbb{I}_{r,i}||(f(W_{t-1},X_i)-Y_i)||_2^2\notag\\
    =&\frac{1}{m}\sum_{r=1}^m2\mathbb{I}_{r,i}\ell(W_{t-1},Z_i).
    \label{ineq:nn-gradient}
\end{align}

In addition, we notice that $\mathrm{Tr}(\mathrm{H}_{W_T}(Z))=\frac{1}{m}\sum_{r=1}^m\mathbb{I}_{r,i,T}$.
Plugging everything into Lemma \ref{lem:sgd-bound-norm}, we have
\[
    \mathrm{gen}(\mu,P_{W_T|S})\leq3\left(\sum_{t=1}^T\frac{R^2\lambda_t^2T}{4nm} \ex{W_{t-1},Z_i}{\sum_{r=1}^m\mathbb{I}_{r,i,t}\ell(W_{t-1},Z_i)}\ex{W_T,Z_i}{\frac{1}{m}\sum_{r=1}^m\mathbb{I}_{r,i,T}}\right)^{\frac{1}{3}}.
\]
This completes the proof.
\end{proof}


\section{Experiment Details}
\subsection{Architectures and Hyperparameters}
In Section \ref{sec:experiment}, MLP has one hidden layer with 512 hidden units, and AlexNet has five  convolution layers (conv. $3\times3$ (64 filters) $\rightarrow$ max-pool $3\times3\rightarrow$ conv. $5\times5$ (192 filters) $\rightarrow$ max-pool $3\times3\rightarrow$ conv. $3\times3$ (384 filters) $\rightarrow$ conv. $3\times3$ (256 filters) $\rightarrow$ conv. $3\times3$ (256 filters) $\rightarrow$ max-pool $3\times3$) followed by two fully connected layers both with 4096 units and a 10-way linear layer as the output layer. All of the convolution layers and the fully connected layers use standard rectified linear activation functions (ReLU).

The fixed learning rates used for MLP and AlexNet are $0.01$ and $0.001$, respectively. The batch size is set to $60$. For the corrupted label experiment, we train the models until the models achieve $100\%$ training accuracy. For other cases, we train the neural networks until the training loss converges (e.g., $<0.0001$). Other settings are either described in Section \ref{sec:experiment} or apparent in the figures. Standard techniques such as weight decay and batch normalization are not used. 

To choose the variance proxy $R$, we first collected all the per-instance losses $\ell(W_{t-1},Z_i)$ that were observed during training, then we let $R=(max_{i,t}\ell(W_{t-1},Z_i)-min_{i,t}\ell(W_{t-1},Z_i))/2$ in our experiments.

In Section \ref{sec:regularization}, we compare GMP with other advanced regularization methods. The results of other methods are reported directly from \cite{zheng2020regularizing}, and we now give their hyperparameter settings here for completeness. For Dropout, $10\%$ of neurons are randomly selected to be deactivated in each layer. For label smoothing, the coefficient is $0.2$. For flooding, the level is set to $0.02$. For MixUp, we lineally combine random pairs of training data where the coefficient is drawn from $\mathrm{Beta}(1, 1)$. For adversarial training, the perturbation size is $1$ for each pixel and we take one step to generate adversarial examples. For AMP, the number of inner iteration is $1$, and the $L_2$ norm ball radius values are $0.5$ for PreActResNet18 and $0.1$ for VGG16, respectively.

The implementation in this paper is on PyTorch, and all the experiments are carried out on NVIDIA Tesla V100 GPUs (32 GB).


\subsection{More Experiments on the Comparison Between Theorem \ref{thm:re-neu-bound} and Lemma \ref{lem:neu-bound}}
\label{appendix:sec-more-compare}
\begin{figure}[ht!]
    \centering
    \includegraphics[width=0.8\textwidth]{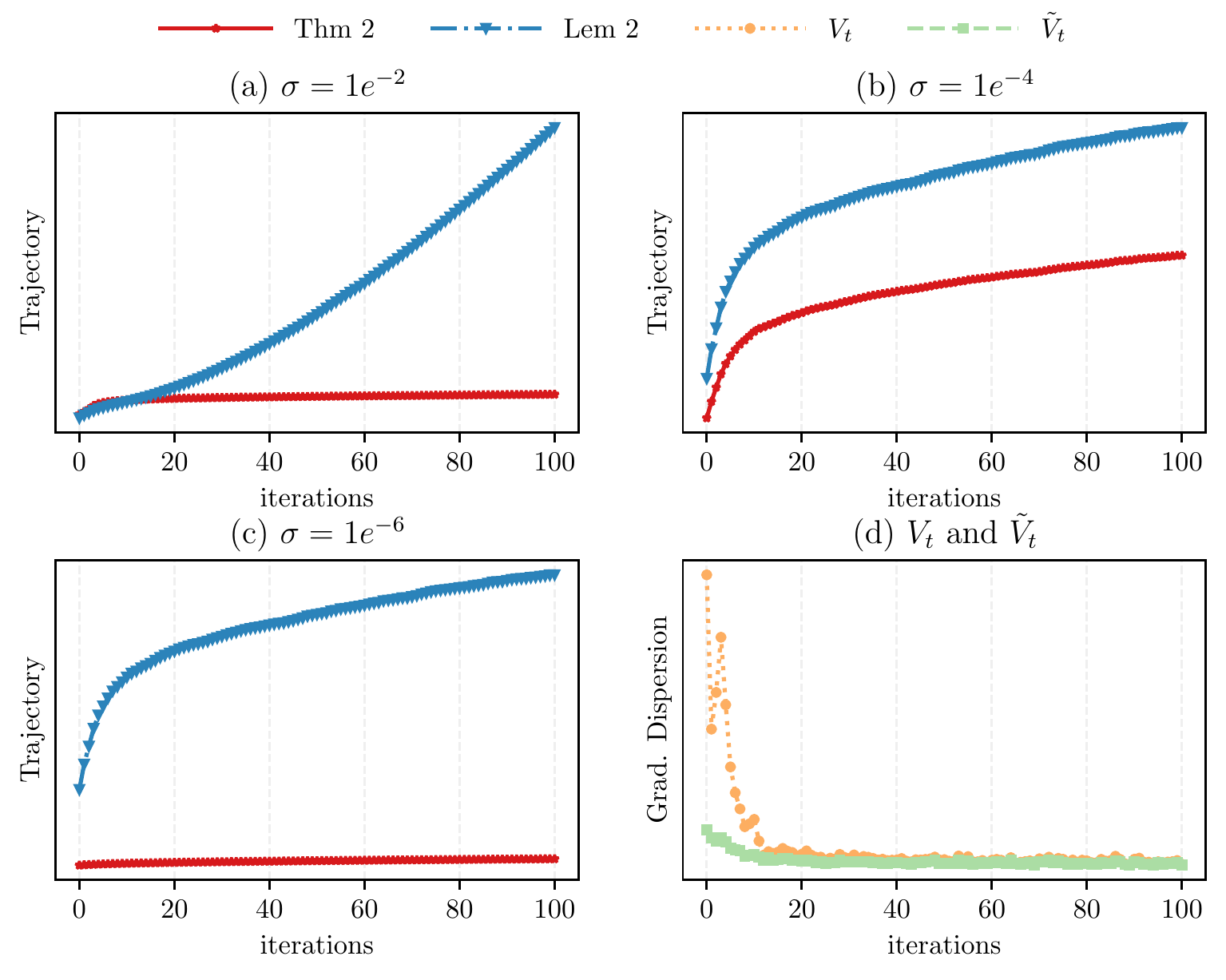}
    \caption{Comparison Between Theorem \ref{thm:re-neu-bound} and Lemma \ref{lem:neu-bound} in stochastic setting.}
\label{fig:CompareMore}
\end{figure}

In Figure \ref{fig:bound-compare}, we compare $\sqrt{\frac{2}{n}\sum_{t=1}^T\ex{}{\Psi(W_{t-1})+ \mathbb{\widetilde{V}}_t(W_{t-1})}}$ and $\sqrt{\frac{1}{n}\sum_{t=1}^T\ex{}{\mathbb{{V}}_t(W_{t-1})}}$ for two different values of $\sigma$. Notice that we indeed use a weaker version of the trajectory term in Theorem \ref{thm:re-neu-bound}, and the same constants like $R$, $\lambda_t$ and $\sigma_t$ in the two bounds are ignored here.
We choose the full dataset of MNIST and CIFAR10 to train the models with the fixed initialization. We also fix the sampling of batches, making the training completely deterministic. To estimate the local gradient sensitivity term, we randomly sample $20$ Gaussian noises from $\mathcal{N}(0,\sum_{\tau=1}^{t-1}\sigma_\tau)$ to perturb the model parameters and compute the average perturbed gradient. With this single draw and deterministic setting, the estimated $\widetilde{\mathbb{V}}_t$ and ${\mathbb{V}}_t$ are the same, so the bound of Eq. \ref{ineq:sgd-bound} in Theorem \ref{thm:re-neu-bound} should be smaller than the bound in Lemma \ref{lem:neu-bound}, as shown in Figure \ref{fig:bound-compare}.

We also provide the experiments under the stochastic setting. Specially, We randomly choose $1/10$ of the MNIST data and train the MLP model with a fixed learning rate and batch size. To estimate $\ex{W_{t-1},Z}{\nabla\ell(W_{t-1},Z)}$, we conduct $20$ runs with different random seeds and save $W$ after every iteration. In Figure \ref{fig:CompareMore} (d), we can see that at the beginning of training, our gradient dispersion $\mathbb{V}_t$ is much larger than $\widetilde{\mathbb{V}}_t$ in \cite{Neu2021InformationTheoreticGB}, but in the later training phase, the magnitude of gap between these two terms is very small. This is because the gradient norm will become tiny when $W$ is near local minima. 
In Figure \ref{fig:CompareMore} (a-c),  we compare $\sqrt{\frac{4}{n}\sum_{t=1}^T\frac{\lambda_t^2 }{\sigma_t^2}\ex{}{\Psi(W_{t-1})+ \mathbb{\widetilde{V}}_t(W_{t-1})}}$ and $\sqrt{\frac{d}{n}\sum_{t=1}^T\log\left(\frac{\lambda_t^2\ex{}{\mathbb{V}_t(W_{t-1})}}{d\sigma_t^2}+1\right)}$ under three different settings of $\sigma$. When $\sigma$ is small (e.g., $\sigma=1e^{-6}$), the local gradient sensitivity term will become small, but the factor $1/\sigma^2$ will be very large, making the gap between $\log(x+1)$ and $x$ be extremely large, as shown in Figure \ref{fig:CompareMore} (c). In this case, the improvement upon \cite{Neu2021InformationTheoreticGB} is significant. When $\sigma$ is large (e.g., $\sigma=1e^{-2}$), at the beginning of training, the trajectory term in Lemma \ref{lem:neu-bound} will be smaller than the trajectory term in Theorem \ref{thm:re-neu-bound} since our $\mathbb{V}_t$ is relatively large. However, since $\Psi(W_{t-1})$ has the cumulative variance, and $\mathbb{V}_t$ and $\widetilde{\mathbb{V}}_t$ become closer soon or later, the bound in Lemma \ref{lem:neu-bound} will be greater than the bound in Theorem \ref{thm:re-neu-bound} in the later training phase.

\subsection{Learning Rate and Batch Size.}
\label{sec:lr-bs}
The learning rate and batch size have some impact on Eq. \ref{eq:optimal-bound} in Theorem \ref{thm:re-neu-bound}. We now investigate this by performing experiments with varying learning rates and batch sizes. In our experiments, the model is continuously updated until the average training loss drops below 0.0001. We separate trajectory and flatness terms of the bound and plot them in Figure \ref{fig:lr-bs}. 

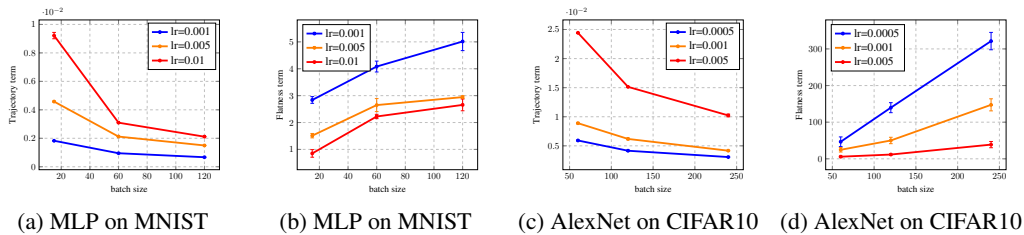
\begin{figure}[ht!]
\begin{subfigure}[b]{0.25\columnwidth}%
\centering%
\captionsetup{font=small}%
\scalebox{0.35}{
\begin{tikzpicture}
\begin{axis}
[legend style={nodes={scale=1.25, transform shape}},
legend cell align={left},
xlabel=batch size,
grid style=dashed,
table/col sep=comma,
xmajorgrids=true,
ymajorgrids=true,
ylabel=Trajectory term,
]
\addplot[mark=star, mark options={scale=1}, line width=2pt, color=blue, error bars/.cd, y dir=both, y explicit] table [y=mv1, x=mbatch, y error=mv1_error]{data/Bound.csv};
\addlegendentry{lr=0.001}
\addplot[mark=star, mark options={scale=1}, line width=2pt, color=orange, error bars/.cd, y dir=both, y explicit] table [y=mv2, x=mbatch, y error=mv2_error]{data/Bound.csv};
\addlegendentry{lr=0.005}
\addplot[mark=star, mark options={scale=1}, line width=2pt, color=red, error bars/.cd, y dir=both, y explicit] table [y=mv3, x=mbatch, y error=mv3_error]{data/Bound.csv};
\addlegendentry{lr=0.01}
\end{axis}
\end{tikzpicture}
}
\caption{MLP on MNIST}%
\label{fig:var-bslr-mnist}
\end{subfigure}%
\begin{subfigure}[b]{0.25\columnwidth}%
\centering%
\captionsetup{font=small}%
\scalebox{0.35}{
\begin{tikzpicture}
\begin{axis}
[legend style={nodes={scale=1.25, transform shape}},
legend cell align={left},
legend pos=north west,
xlabel=batch size,
grid style=dashed,
table/col sep=comma,
xmajorgrids=true,
ymajorgrids=true,
ylabel=Flatness term,
]
\addplot[mark=star, mark options={scale=1}, line width=2pt, color=blue, error bars/.cd, y dir=both, y explicit] table [y=mh1, x=mbatch, y error=mh1_error]{data/Bound.csv};
\addlegendentry{lr=0.001}
\addplot[mark=star, mark options={scale=1}, line width=2pt, color=orange, error bars/.cd, y dir=both, y explicit] table [y=mh2, x=mbatch, y error=mh2_error]{data/Bound.csv};
\addlegendentry{lr=0.005}
\addplot[mark=star, mark options={scale=1}, line width=2pt, color=red, error bars/.cd, y dir=both, y explicit] table [y=mh3, x=mbatch, y error=mh3_error]{data/Bound.csv};
\addlegendentry{lr=0.01}
\end{axis}
\end{tikzpicture}
}
\caption{MLP on MNIST}%
\label{fig:flat-bslr-mnist}
\end{subfigure}%
\begin{subfigure}[b]{0.25\columnwidth}%
\centering%
\captionsetup{font=small}%
\scalebox{0.35}{
\begin{tikzpicture}
\begin{axis}
[legend style={nodes={scale=1.25, transform shape}},
legend cell align={left},
xlabel=batch size,
grid style=dashed,
table/col sep=comma,
xmajorgrids=true,
ymajorgrids=true,
ylabel=Trajectory term,
]
\addplot[mark=star, mark options={scale=1}, line width=2pt, color=blue, error bars/.cd, y dir=both, y explicit] table [y=cv1, x=cbatch, y error=cv1_error]{data/Bound.csv};
\addlegendentry{lr=0.0005}
\addplot[mark=star, mark options={scale=1}, line width=2pt, color=orange, error bars/.cd, y dir=both, y explicit] table [y=cv2, x=cbatch, y error=cv2_error]{data/Bound.csv};
\addlegendentry{lr=0.001}
\addplot[mark=star, mark options={scale=1}, line width=2pt, color=red, error bars/.cd, y dir=both, y explicit] table [y=cv3, x=cbatch, y error=cv3_error]{data/Bound.csv};
\addlegendentry{lr=0.005}
\end{axis}
\end{tikzpicture}
}
\caption{AlexNet on CIFAR10}%
\label{fig:var-bslr-cifar}
\end{subfigure}%
\begin{subfigure}[b]{0.25\columnwidth}%
\centering%
\captionsetup{font=small}%
\scalebox{0.35}{
\begin{tikzpicture}
\begin{axis}
[legend style={nodes={scale=1.25, transform shape}},
legend cell align={left},
legend pos=north west,
xlabel=batch size,
grid style=dashed,
table/col sep=comma,
xmajorgrids=true,
ymajorgrids=true,
ylabel=Flatness term,
]
\addplot[mark=star, mark options={scale=1}, line width=2pt, color=blue, error bars/.cd, y dir=both, y explicit] table [y=ch1, x=cbatch, y error=ch1_error]{data/Bound.csv};
\addlegendentry{lr=0.0005}
\addplot[mark=star, mark options={scale=1}, line width=2pt, color=orange, error bars/.cd, y dir=both, y explicit] table [y=ch2, x=cbatch, y error=ch2_error]{data/Bound.csv};
\addlegendentry{lr=0.001}
\addplot[mark=star, mark options={scale=1}, line width=2pt, color=red, error bars/.cd, y dir=both, y explicit] table [y=ch3, x=cbatch, y error=ch3_error]{data/Bound.csv};
\addlegendentry{lr=0.005}
\end{axis}
\end{tikzpicture}
}
\caption{AlexNet on CIFAR10}%
\label{fig:flat-bslr-cifar}
\end{subfigure}%
\caption{ The impact of learning rate and batch size on the trajectory term and the flatness term in Eq. \ref{eq:optimal-bound}}
\label{fig:lr-bs}
\end{figure}


A key observation in Figure \ref{fig:lr-bs}  is that the learning rate impacts the trajectory term and the flatness term in opposite ways, as seen, for example, in (a) and (b), where the two set of curves swap their orders in the two figures. On the other hand, the batch size also impacts the two terms in opposite ways, as seen in (a) and (b) where curves decrease in (a) but increase in (b). This makes the generalization bound, i.e., the sum of the two terms, have a rather complex relationship with the settings of learning rate and batch size. This relationship is further complicated by the fact that a small learning rate requires a longer training time, or a larger number $T$ of training iterations, which increases the number that are summed over in the trajectory term.
Nonetheless, we do observe that a smaller batch size gives a lower value of the flatness term ((b) and (d)), confirming the previous wisdom that small batch sizes enable the neural network to find a flat minima \citep{keskar2017large}.

\begin{figure}[ht!]
    \centering
    \includegraphics[width=0.4\textwidth]{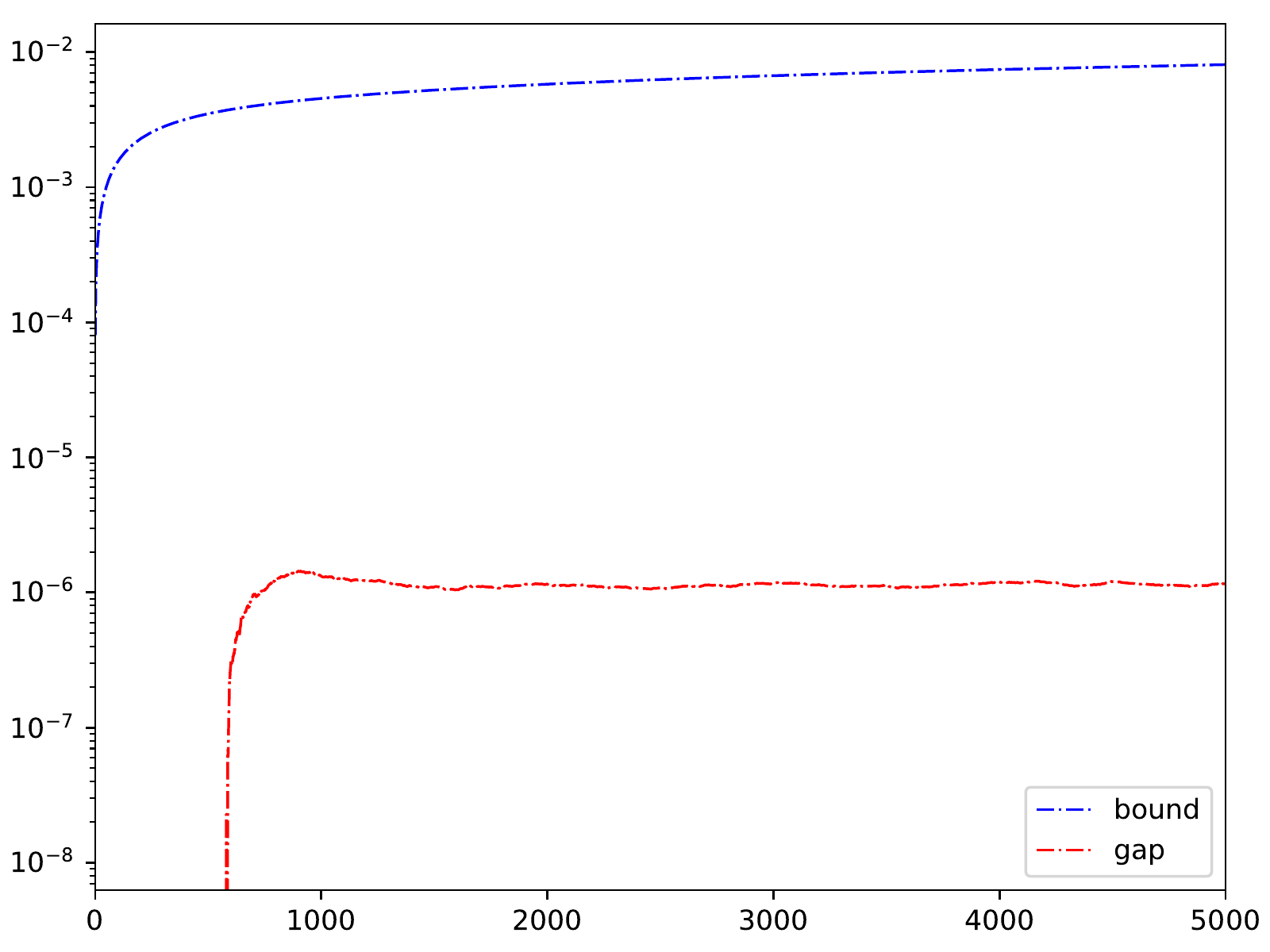}
    \caption{Generalization gap of a linear network v.s. Theorem \ref{thm:optimal-bound-linear}. Note y-axe is log-scale.}
\label{fig:LNbound}
\end{figure}
\begin{figure}[ht!] 
    \centering
    \includegraphics[width=0.4\textwidth]{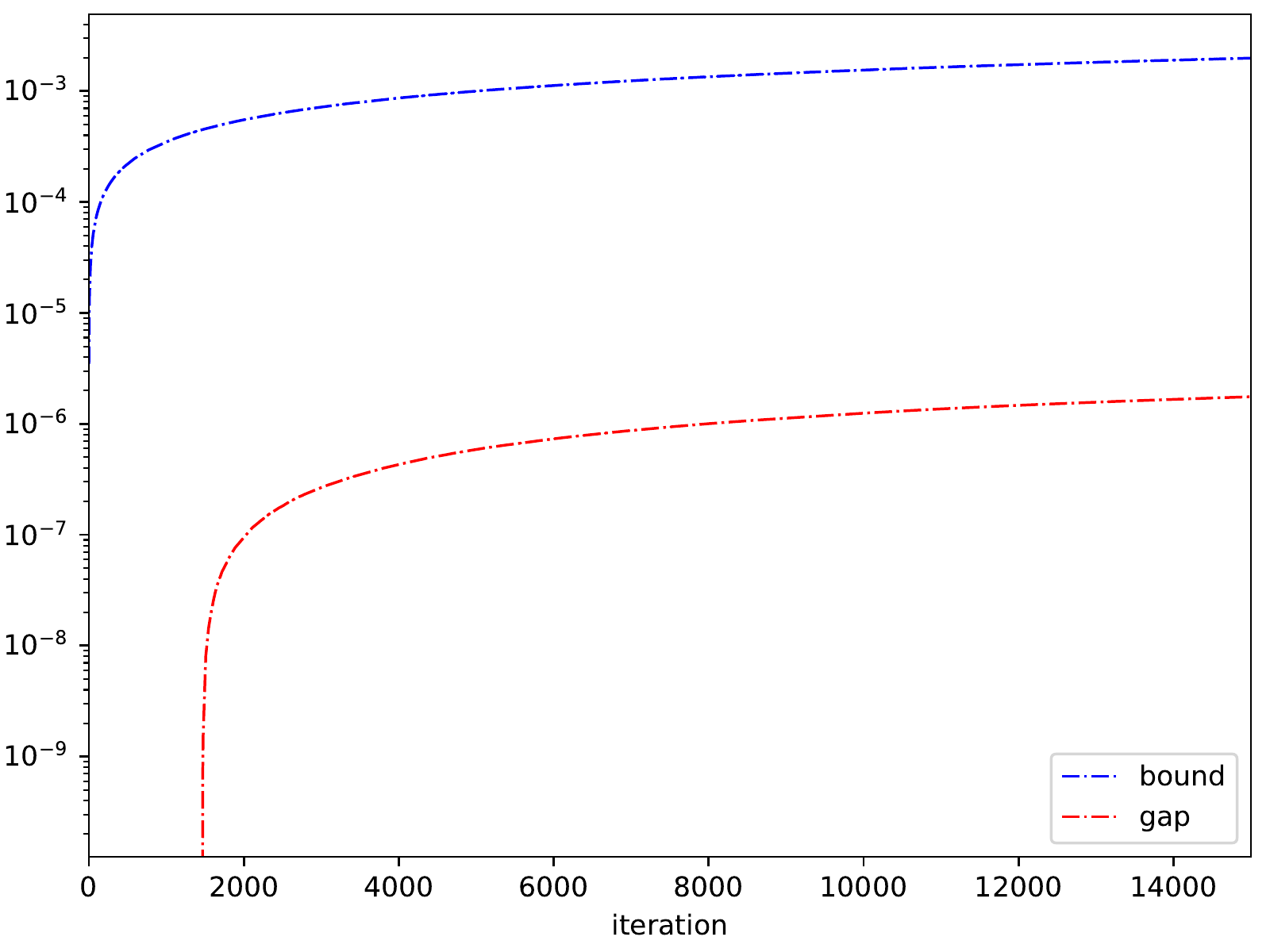}
    \caption{Generalization gap of a two-layer ReLU network v.s. Theorem \ref{thm:optimal-bound-relu}.}
\label{fig:NNbound}
\end{figure}

\subsection{Implementation of Theorem \ref{thm:optimal-bound-linear} and Theorem \ref{thm:optimal-bound-relu}}
\label{appendix:sec-NN-bound}

We let $d_0=200$ and use a two-layer ReLU network with hidden units $10000$ to generate $Y$. Moreover, we apply a $\tanh$ function to the output of this network so that $|Y|\leq1$. The input data $X\sim\mathcal{N}(0,\mathrm{I})$ and we normalize $X$ before training so that $||X||=1$. 
For the training phase, we choose $m=100$, let data size be $20,000$, batch size be $100$ and learning rate be $0.5$. Figure \ref{fig:LNbound} and Figure \ref{fig:NNbound} compare the empirical generalization gap with the bound in Theorem \ref{thm:optimal-bound-linear} and that in Theorem \ref{thm:optimal-bound-relu}, respectively.

\subsection{Algorithm of Dynamic Gradient Clipping and Additional Results}
\label{sec:DYG-appendix}

The dynamic gradient clipping algorithm is described in Algorithm \ref{alg:dygc}. For both MLP and AlexNet, we let $\alpha = 0.1$. The start step for clipping, $T_c$, is also an important hyperparameter. However, it can be removed by detecting the evolution of the average gradient norm for each epoch. Specifically, whenever the average gradient norm of epoch $j$ is larger than the average gradient norm of epoch $j-1$, the clipping operation begins.

\begin{algorithm}[!htbp]
\caption{Dynamic Gradient Clipping}
\label{alg:dygc}
\begin{algorithmic}[1]
\REQUIRE Training set $S$, Batch size $b$, Loss function $\ell$, Initial model parameter $\boldsymbol{w}_0$, Learning rate $\lambda$, Initial minimum gradient norm $\mathcal{G}$, Number of iterations $T$, Clipping parameter $\alpha$, Clipping step $T_c$
\FOR{$t\gets1\text{ to }T$}
\STATE Sample $\mathcal{B}=\{\boldsymbol{z}_i\}_{i=1}^b$ from training set $S$
\STATE Compute gradient: \\ $g_{\mathcal{B}}\gets\sum_{i=1}^b\nabla_{\boldsymbol{w}}\ell(\boldsymbol{w}_{t-1},\boldsymbol{z}_i)/b$
\IF{$t>T_c$}
\IF{$||g_{\mathcal{B}}||_2>\mathcal{G}$}
    \STATE $g_{\mathcal{B}}\gets \alpha\cdot\mathcal{G}\cdot g_{\mathcal{B}}/||g_{\mathcal{B}}||_2$
\ELSE
    \STATE $\mathcal{G}\gets||g_{\mathcal{B}}||_2$
\ENDIF
\ENDIF
\STATE Update parameter: $\boldsymbol{w}_{t}\gets\boldsymbol{w}_{t-1}-\lambda\cdot g_{\mathcal{B}}$
\ENDFOR
\end{algorithmic}
\end{algorithm}

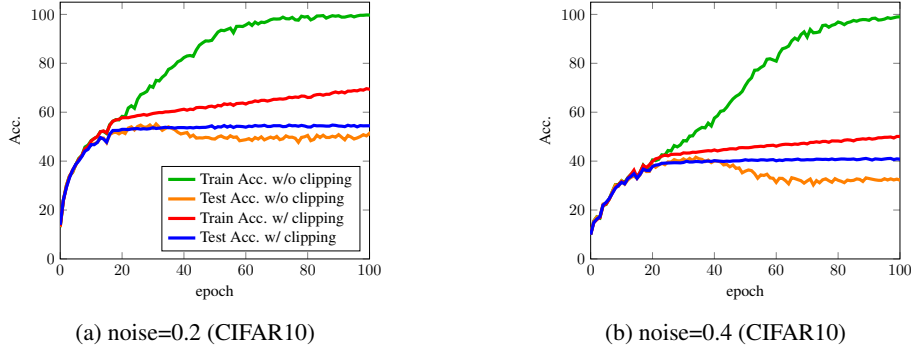
\begin{figure}[ht!]
\begin{subfigure}[b]{0.5\columnwidth}%
\centering%
\captionsetup{font=small}%
\scalebox{0.6}{
\begin{tikzpicture}
\begin{axis}
[
legend cell align={left},
legend pos=south east,
ymin=0, ymax=105,
xlabel=epoch,
grid style=dashed,
xmin=0,
    xmax=100,
ylabel=Acc.
]
\addplot[line width=2pt, color=black!30!green] table [y=TrAcc1, x=epoch0]{data/Alex_old.txt};
\addlegendentry{Train Acc. w/o clipping}
\addplot[line width=2pt, color=orange] table [y=TsAcc1, x=epoch0]{data/Alex_old.txt};
\addlegendentry{Test Acc. w/o clipping}
\addplot[line width=2pt, color=red] table [y=TrAcc, x=epoch]{data/Alex02dc.txt};
\addlegendentry{Train Acc. w/ clipping}
\addplot[line width=2pt, color=blue] table [y=TsAcc, x=epoch]{data/Alex02dc.txt};
\addlegendentry{Test Acc. w/ clipping}
\end{axis}
\end{tikzpicture}
}
\caption{noise=0.2 (CIFAR10)}%
\end{subfigure}
\begin{subfigure}[b]{0.5\columnwidth}%
\centering%
\captionsetup{font=small}%
\scalebox{0.6}{
\begin{tikzpicture}
\begin{axis}
[
legend cell align={left},
legend pos=south east,
ymin=0, ymax=105,
xlabel=epoch,
grid style=dashed,
xmin=0,
    xmax=100,
ylabel=Acc.
]
\addplot[line width=2pt, color=black!30!green] table [y=TrAcc2, x=epoch0]{data/Alex_old.txt};

\addplot[line width=2pt, color=orange] table [y=TsAcc2, x=epoch0]{data/Alex_old.txt};

\addplot[line width=2pt, color=red] table [y=TrAcc, x=epoch]{data/Alex04dc.txt};

\addplot[line width=2pt, color=blue] table [y=TsAcc, x=epoch]{data/Alex04dc.txt};

\end{axis}
\end{tikzpicture}
}
\caption{noise=0.4 (CIFAR10)}%
\end{subfigure}
\caption{Dynamic Gradient Clipping (AlexNet).}
\label{fig:DyGC-alex}
\end{figure}

 From Figure \ref{fig:DyGC} and Figure \ref{fig:DyGC-alex} we can see that dynamic gradient clipping effectively alleviates overfitting by conspicuously slowing down the transition of training to the memorization regime, without changing the convergence speed of testing accuracy. 
 Unfortunately, the current design of the dynamic gradient clipping algorithm does not provide a significant improvement for models trained on a clean dataset (without label noise). Designing better regularization algorithms may require understanding the epoch-wise double descent curve of gradient dispersion where the model is trained on a clean dataset.
 
 \subsection{Discussion on Gradient Dispersion of Models Trained on Clean Datasets}
In the case of no noise injected, Figure \ref{fig:acc-var-mlp} shows that the model with good generalization property has a exponentially-decaying gradient dispersion. This is consistent with our discussion of Lemma \ref{lem:sample-cmi-var} in Section \ref{sec:theory}, that is, small $I(G_t+N_t;S|\widetilde{W}_{t-1})$ indicates good generalization. Notably, gradient dispersion of AlexNet trained on the real CIFAR10 data still has a epoch-wise double descent curve. The difference between Figure \ref{fig:acc-var-alex} with Figure \ref{fig:acc-var-alex02}-\ref{fig:acc-var-alex06} is that the testing accuracy does not decrease in the second phase/memorization regime for AlexNet trained on the true CIFAR10 data.  Loosely speaking, we conjugate that memorizing random labels will hurt the performance on unseen clean data but memorizing clean (or true) labels will not. This may explain why dynamic gradient clipping or preventing the training entering the memorization regime cannot improve the performance on a clean dataset.

\subsection{Algorithm of Gaussian Model Perturbation and Additional Results}
\label{sec:gmp-appendix}
\begin{algorithm}[!htbp]
\caption{Gaussian Model Perturbation Training}
\label{alg:gmp}
\begin{algorithmic}[2]
\REQUIRE Training set $S$, Batch size $b$, Loss function $\ell$, Initial model parameter $\boldsymbol{w}_0$, Learning rate $\lambda$, Number of noise $k$, Standard deviation of Gaussian distribution $\sigma$, Lagrange multiplier $\rho$ 
\WHILE{$\boldsymbol{w}_t$ not converged}
\STATE Update iteration: $t\gets t+1$
\STATE Sample $\mathcal{B}=\{\boldsymbol{z}_i\}_{i=1}^b$ from training set $S$
\STATE Sample $\Delta_j\sim\mathcal{N}(0,\sigma^2)$ for $j\in [k]$
\STATE Compute gradient: \\ 
$g_{\mathcal{B}} \gets \sum_{i=1}^b\left(
\nabla_{\boldsymbol{w}}\ell(\boldsymbol{w}_t,z_i)+\rho\sum_{j=1}^k\left(\nabla_{\boldsymbol{w}}\ell(\boldsymbol{w}_t+\Delta_j,z_i)-\nabla_{\boldsymbol{w}}\ell(\boldsymbol{w}_t,z_i)\right)/k\right)/b$
\STATE Update parameter: $\boldsymbol{w}_{t+1}\gets\boldsymbol{w}_{t}-\lambda\cdot g_{\mathcal{B}}$
\ENDWHILE
\end{algorithmic}
\end{algorithm}

The GMP algorithm is given in Algorithm \ref{alg:gmp}. Table \ref{tab:cvresults} and Table \ref{tab:cvresults-appendix} show that our method is competitive to the state-of-the-art regularization techniques. Specifically, our method achieves the best performance on SVHN for both models and on CIFAR-100 where VGG16 is employed. Particularly, testing accuracy is improved by nealy $2\%$ on the CIFAR-100 dataset with VGG16. For other tasks, GMP is always able to achieve the top-3 performance. In addition, we find that increasing the number of sampled noises does not guarantee the improvement of testing accuracy and may even degrade the performance on some datasets (e.g., SVHN). This hints that we can use small number of noises to reduce the running time without losing performance. 
Moreover, we observe that \textbf{GMP with $k=3$ usually takes around $1.76\times$ that of ERM training time}, which is affordable. 

For PreActResNet18, the performance of GMP appears slighly inferior to 
the most recent record given by 
AMP \citep{zheng2020regularizing}.  Noting that the key ingredient of AMP,  ``max-pooling'' in the parameter space, greatly resembles regularization term in GMP, which may be seen as ``average-pooling'' in the same space. 

One potential extension of GMP is to let the variance of the noise distribution be a function of the iteration step $t$. In other words, using the time-dependent $\sigma_t$ instead of a constant $\sigma$.

\begin{table*}[t]
\centering
\begin{tabular}{l|c|c|c}
\toprule
Method & SVHN & CIFAR-10 & CIFAR-100\\
\midrule
ERM & 97.05$\pm$0.063  & 94.98$\pm$0.212 & 75.69$\pm$0.303 \\
Dropout & 97.20$\pm$0.065  & 95.14$\pm$0.148 & 75.52$\pm$0.351\\
L.S. & 97.22$\pm$0.087  & 95.15$\pm$0.115 & 77.93$\pm$0.256 \\
Flooding & 97.16$\pm$0.047 & 95.03$\pm$0.082 & 75.50$\pm$0.234  \\
MixUp & 97.26$\pm$0.044 & \underline{95.91$\pm$0.117} & \underline{78.22$\pm$0.210}\\
Adv. Tr. & 97.23$\pm$0.080 & 95.01$\pm$0.085 & 74.77$\pm$0.229  \\
AMP & \textbf{97.70$\pm$0.025} &\textbf{96.03$\pm$0.091}  & \textbf{78.49$\pm$0.308}  \\
\textbf{GMP$^3$} & \underline{97.43$\pm$0.037} & {95.64$\pm$0.053}   & 78.05$\pm$0.208  \\
\textbf{GMP$^{10}$} & {97.34$\pm$0.058} & {95.71$\pm$0.073} & {78.07$\pm$0.170}  \\
\bottomrule
\end{tabular}
\caption{Top-1 classification accuracy acc.($\%$) of PreActResNet18. 
We run experiments 10 times and report the mean and the standard deviation of the testing accuracy. 
Superscript denotes the number of sampled Gaussian noises during training.
}

\label{tab:cvresults-appendix}
\end{table*}

We also consider the following regularized scheme, which is a absolute value version.
\[
\min_{w} L_s(w) + \rho\ex{\Delta\sim\mathcal{N}(0,\sigma^2\mathbf{I}_d)}{\Big\lvert L_s(w+\Delta)-L_s(w)\Big\rvert}.
\]
This scheme can still perform well as shown in Table \ref{tab:cvresults-2-appendix}. In fact, it outperforms GMP$^3$ on CIFAR-100. This hints that it is possible to improve the performance by choosing other norm of $L_s(w+\Delta)-L_s(w)$.

\begin{table*}[t]
\centering
\begin{tabular}{l|c|c|c}
\toprule
Method & SVHN & CIFAR-10 & CIFAR-100\\
\midrule
\textbf{GMP$_{\rm abs}^3$} & {97.10$\pm$0.054
} & {94.21$\pm$0.139}   & 74.80$\pm$0.113  \\
\bottomrule
\end{tabular}
\caption{Top-1 classification accuracy acc.($\%$) of VGG16. 
}
\label{tab:cvresults-2-appendix}
\end{table*}

\end{document}